\documentclass[11pt]{article}
\usepackage{amssymb, amsmath, amsthm, amsfonts, bbm}
\usepackage{algorithm, algpseudocode}

\usepackage{graphicx,caption,subcaption}
\usepackage{natbib}
\usepackage{color}
\usepackage{xcolor}
\usepackage[colorlinks,linkcolor=blue,citecolor=blue!70!black,urlcolor=black]{hyperref}

\usepackage{fullpage}

\usepackage{tikz}
\usetikzlibrary{arrows,arrows.meta}

\newtheorem*{theorem*}{Theorem}
\newtheorem{theorem}{Theorem}
\newtheorem{fact}{Fact}

\newtheorem{proposition}{Proposition}
\newtheorem{corollary}{Corollary}

\newtheorem{definition}{Definition}
\newtheorem{example}{Example}

\usepackage[capitalize,noabbrev]{cleveref}
\crefname{theorem}{Theorem}{Theorems}
\crefname{lemma}{Lemma}{Lemmas}
\crefname{corollary}{Corollary}{Corollaries}
\crefname{proposition}{Proposition}{Propositions}
\crefname{definition}{Definition}{Definitions}
\crefname{example}{Example}{Examples}
\crefname{remark}{Remark}{Remarks}
\crefname{assumption}{Assumption}{Assumptions}
\crefname{fact}{Fact}{Facts}
\crefname{result}{Result}{Results}
\crefname{problem}{Problem}{Problems}
\crefname{question}{Question}{Questions}

\usepackage{mathtools}

\newcommand{\R}{\mathbb{R}}  
\newcommand{\E}{\mathbb{E}}
\newcommand{\cA}{\mathcal{A}}

\newcommand{\cF}{\mathcal{F}}
\newcommand{\cH}{\mathcal{H}}
\newcommand{\cX}{\mathcal{X}}
\newcommand{\cD}{\mathcal{D}}
\newcommand{\oigap}{\mathrm{OIGap}}
\newcommand{\cP}{\mathcal{P}}
\newcommand{\yt}{\widetilde{y}}
\newcommand{\cY}{\mathcal{Y}}

\newcommand{\Tr}{\mathrm{Tr}}
\newcommand{\poly}{\mathrm{poly}}

\newcommand{\cZ}{\mathcal{Z}}
\newcommand{\cO}{\mathcal{O}}

\renewcommand{\epsilon}{\varepsilon}
\newcommand{\eps}{\epsilon}

\newcommand{\OIalg}{\text{Defensive Generation}}
\usepackage{xcolor}

\algnewcommand\algorithmicinput{\textbf{Input:}}
\algnewcommand\Input{\item[\algorithmicinput]}
\usepackage{amsmath}
\usepackage{amsthm}
\usepackage{mismath}

\DeclareMathOperator*{\argmin}{arg\,min}

\begin{document}

\title{Defensive Generation} 
\author{Gabriele Farina  \\MIT \\ \url{gfarina@mit.edu} \and Juan Carlos Perdomo \\ MIT, NYU \\ \url
{j.perdomo.silva@nyu.edu}}

\maketitle

\begin{abstract}\noindent
We study the problem of efficiently producing, in an online fashion, generative models of scalar, multiclass, and 
vector-valued outcomes that cannot be falsified on the basis of the observed data and a pre-specified collection of computational tests. 
Our contributions are twofold. First, we expand on connections between online high-dimensional multicalibration with respect to an RKHS and recent advances in expected variational inequality problems, enabling efficient algorithms for the former. We then apply this algorithmic machinery to the problem of outcome indistinguishability.
Our procedure, Defensive Generation, is the first to efficiently produce online outcome indistinguishable generative models of non-Bernoulli outcomes that are unfalsifiable with respect to infinite classes of tests, including those that examine higher-order moments of the generated distributions. Furthermore, our method runs in near-linear time in the number of samples and achieves the optimal, vanishing $1/\sqrt T$ rate for generation error.
\end{abstract}

\section{Introduction}
Supervised learning frames learning in terms of loss minimization. 
Under this paradigm, an algorithm succeeds if it is able to find a prediction rule that has low excess risk over the underlying data distribution, with respect to some loss and hypothesis class.
Drawing on a rich set of ideas from complexity theory and cryptography, outcome indistinguishability offers a different perspective \citep{dwork2021outcome}. 
A learner succeeds if it finds a generative model of outcomes that cannot be falsified on the basis of the observed data and a pre-specified collection of computational tests. 

The motivation behind outcome indistinguishability is that if no computational process can tell the difference between the ``true'' data produced by Nature and data produced by the Learner's model, then this model effectively \emph{is} the ``real'' data generating process. 
The validity of the Learner's model is staked on its ability to withstand falsification.

In this paper, we study the problem of designing computationally efficient algorithms that provably find outcome indistinguishable generative models for rich classes of outcomes. We study this question in the online setting where  data is arbitrarily (possibly adversarially) chosen. Samples are not assumed to be drawn i.i.d. from any distribution, yet the Learner is tasked with finding a generative model that, for all intents and purposes, ``looks like'' it generated the observed sequence.

Our main contribution is an efficient online procedure that finds generative models of scalar, multiclass, or vector-valued outcomes that are provably outcome indistinguishable with respect to rich, infinite collections of tests that live in a vector-valued reproducing kernel Hilbert space. 
A generative model here is a function that given a set of features $x$ produces a conditional distribution $\mu$ over outcomes $\cY$.
Prior work in this area was either restricted to the case of finding generative models of simple binary outcomes, or required explicit enumeration over the set of tests (which ruled out efficiently guaranteeing indistinguishability with respect to super polynomially-sized classes). Our algorithm not only expands the scope of settings where one can guarantee indistinguishability, it also achieves the optimal $\sqrt{T}$ regret bound for this problem (see \citet{vovk2007k29} for a lower bound). 

On a technical level, our result comes from reducing online outcome indistinguishability to high-dimensional multicalibration and using recent advances on expected variational inequalities (EVIs) to efficiently solve these underlying calibration problems. A core part of our work hence relies on exploring the spiderweb of connections between indistinguishability, online calibration, research on learning in games, and nonlinear optimization.

On a conceptual level, we highlight how our work provides a different perspective on generative modeling, based on online learning as its foundation.
Much of the recent theoretical literature assumes that there is a fixed distribution we wish to sample from, and relies on function approximation assumptions to guarantee that the learned model is close in statistical distance to the truth. Our work on online OI, on the other hand, does not guarantee closeness in statistical distance nor does it rely on unverifiable function approximation assumptions. Instead, it provides an end-to-end unconditional guarantee that the learned model is \emph{computationally} indistinguishable from the ``truth''; where truth is in quotation marks since data is arbitrary and there is no underlying distribution to speak of.

Lastly, we produce these indistinguishable models through defensive forecasting, an algorithmic methodology pioneered in \cite{vovk2005defensive}. Rather than trying to make a good guess around what the true data will be, defensive forecasting views predictions as a game. A good prediction is not one which aims to mimic Nature, but rather one that ensures that the generative model looks good in hindsight (from the perspective of the set of tests) \emph{no matter} the choice of Nature.  

We provide an overview of our contributions in \cref{sec:contributions}.
The interested reader can skip ahead to \Cref{sec:examples} to see example guarantees of our Defensive Generation algorithm in domains like language generation, learning linear dynamical systems, and weather forecasting.
A technical overview is given in \cref{sec:techniques}.

\subsection{Overview of Contributions}
\label{sec:contributions}

We design algorithms that work in the following online protocol. At every time step $t$, the Learner sees features $x_t$ chosen by Nature and randomly produces a distribution $\mu_t$ over the outcome space $\cY$ and an auxiliary vector of statistics $p_t$ (think $p_t=g(\mu_t)$ for some function $g$). Lastly, Nature reveals an outcome $y_t \in \cY$. 
Our goal is to design algorithms which achieve the following desideratum.
\begin{definition}
An algorithm $\cA$ satisfies online outcome indistinguishability with respect to a class of distinguishers $\cF$ if it produces distributions $\mu_t$ over $\cY$ and statistics $p_t$ such that for any $f\in \cF$,
\begin{align}
\oigap_T(f) \coloneqq \bigg|    \sum_{t=1}^T \E_{p_t}[f(x_t, p_t, y_t)] - \sum_{t=1}^T \E_{p_t, \yt_t \sim \mu_t} [f(x_t, p_t, \yt_t)] \bigg|,
\label{eq:oigap}
\end{align}
is $o(T)$ regardless of Nature's choices of $(x_t,y_t)$. We define $\oigap_T \coloneqq \sup_{f \in \cF} \oigap_T(f)$.
\end{definition}
On the left, we have the outputs of the distinguishers $f$ on the real outcomes $y_t$. On the right, there is the expected value of the distinguisher over the outcomes $\yt_t \sim \mu_t$ sampled by the Learner's model. Dividing both sides of the above equation by $T$, an algorithm guarantees online outcome indistinguishability if the value of every function in $\cF$ is the same when: $(1)$ averaged over the realized sequence of outcomes $y_t$, or $(2)$ the simulated outcomes $\yt_t$. No distinguisher can spot a difference between the real data and simulated data:
\begin{align*}
\lim_{T\rightarrow \infty}\bigg|    \frac{1}{T} \sum_{t=1}^T \E_{p_t}[f(x_t, p_t, y_t)] - \frac{1}{T} \sum_{t=1}^T \E_{p_t, \yt_t \sim \mu_t}[f(x_t, p_t, \yt_t)] \bigg| \rightarrow 0.
\end{align*}
As a thought experiment, if the outcomes $y_t$ were truly random and drawn from $\mu_t$, by a martingale argument, we should expect that for any $f$, $\oigap_T(f) \approx \sqrt{T}$. Our main contribution is a new algorithm, $\OIalg$, that achieves exactly the same guarantee, online, and for broad classes of adversarially chosen outcomes $y_t$.
\begin{theorem}[Informal]
The following are true regarding the $\OIalg$ procedure. Furthermore, in each setting, it runs in time $\cO(\poly(d) \cdot \log(t))$ at time $t$ and has $\oigap_T \leq \cO(\sqrt{T})$.
 \begin{enumerate}
	\item \underline{Multiclass outcomes}, $\cY = [d]$. The procedure outputs distributions $\mu_t$ that are online outcome indistinguishable with respect to any infinite collection of distinguishers $f(x,p,y)$ that form an RKHS, and which have full access to the entire conditional distribution $\mu_t$ ({i.e.}, $p_t=\mu_t$). 

    As a specific example, this in particular implies indistinguishability with respect to all linear functions at a rate bounded by $d\sqrt{T}$. See \Cref{example:multiclass}.

	\item \underline{Scalar outcomes}, $\cY = [-1,1]$. $\OIalg$ produces distributions $\mu_t$ that are online outcome indistinguishable with respect to all low-degree tests in any RKHS that only examine the first $d$ moments of $\mu_t$. That is, $p_t = (1,\E_{\mu_t}[\yt_t], \E_{\mu_t}[\yt_t^{\;2}], \dots , \E_{\mu_t}[\yt_t^{\;d}])$.

    For Boolean $x$, this implies that one can produce distributions over scalar $\yt$ whose first $d$ conditional moments, $\E[\yt^{\;j}|c(x) =1]$, for $j=1$ to $d$, match those of the observed sequence over all subsets of the hypercube computable by low-depth decision trees $c(x)$. See \Cref{example:scalar}.

	\item \underline{High-dimensional outcomes}, $\cY = \{y\in \R^d: \|y\| \leq 1\}$. It outputs $\mu_t$ that are online OI with respect to distinguishers in any RKHS that examine the mean and covariance  of $\mu_t$. That is, $$p_t= (\E_{\mu_t}[\yt_t], \E_{\mu_t}[\yt_t \yt_t^\top]).$$

    This in particular implies indistinguishability with respect to the infinite set of distinguishers that can lie in the span of a fixed set of (nonlinear) features $\Phi(x,p)$. See \Cref{example:vector2}.

    Moreover, if the distinguishers only examine the first moments, $p=\E_{\mu_t}[\yt_t]$, then the Defensive Generation algorithm works for any compact convex set in $\R^d$, not just the unit ball. See \Cref{example:vector}.
\end{enumerate}
\end{theorem}

To the best of our knowledge, this is the first algorithm that can efficiently guarantee online outcome indistinguishability with respect to infinite classes $\cF$ beyond the case of Bernoulli outcomes.
As mentioned previously, the main backbone of this result is a new near-linear-time algorithm for online multicalibration in high-dimensions. We state this problem formally:
\begin{definition}
Assume that at every time step, Nature selects features $x_t \in \cX$ arbitrarily and then the Learner randomly samples a forecast $p_t$ in a compact, convex set $\cZ \subset \R^d$. Lastly, Nature reveals the target $z_t \in \cZ$.
An algorithm $\cA$ guarantees online multicalibration with respect to a class of functions $\cH \subseteq \{\cX \times \cZ\} \rightarrow \R^d$ if it randomly produces a sequence of $p_t$ such that 
\begin{align*}
\sup_{h \in \cH}\bigg|    \sum_{t=1}^T \E_{p_t} \bigr[ h(x_t,p_t)^\top (z_t - p_t) \bigr] \bigg| \leq o(T) 
\end{align*}
regardless of Nature's choices of $(x_t,z_t) \in \cX \times \cZ$ in this online protocol.
\end{definition}
Unlike other versions of online $(\ell_1)$ calibration studied in the literature \citep{peng2025high,fishelson2025highdimensional} where errors are summed up over a grid of predicted values $p_t=v$, one can achieve error $\eps$ for this multicalibration problem after $\cO(\eps^{-2})$ many rounds for rich classes of functions $\cH$. 

This is in contrast to the lower bound of $d^{\;\poly(1/\eps)}$ many rounds for the $\ell_1$ version. In particular, we design an algorithm that guarantees $\sqrt{T}$ regret for any vector-valued RKHS $\cH$. Our construction relies on recent advances in solving expected variational inequality problems \citep{zhang2025expected}, along with the high-dimensional defensive forecasting algorithm from \cite{kernelOI}. We state its guarantees below:
\begin{theorem}[Informal]
Let $\cH \subseteq \{\cX \times \cZ \rightarrow \R^d\}$ be any vector-valued RKHS with corresponding matrix valued kernel $\Gamma$. Then, the high-dimensional defensive forecasting algorithm can be efficiently implemented to run in time $\cO(\poly(d) \log(t))$ at time step $t$ and guarantee that for any $h \in \cH$,
\begin{align*}
\bigg|    \sum_{t=1}^T \E_{p_t} \bigr[ h(x_t,p_t)^\top (z_t - p_t) \bigr] \bigg| \leq \|h\|_{\cH} \sqrt{\sum_{t=1}^T\E_{p_t}(z_t - p_t)^\top\Gamma((x_t,p_t), (x_t,p_t))(z_t - p_t)}.
\end{align*}
Here, $\|\cdot\|_{\cH}$ denotes the norm of the functions in the RKHS $\cH$.
\end{theorem}

As a final note, prior work \citep{gopalan2022loss,kernelOI,okoroafor2025near}  has established that models that are OI are also loss minimizing. OI in fact implies loss minimization not just for a single loss, but rather for many losses simultaneously (that is, omniprediction) \citep{omnipredictors}. By efficiently guaranteeing OI in higher-dimensions, we hope to enable future work on faster algorithms for omniprediction in richer domains. 

\subsection{Example Applications}\label{sec:examples}

We provide in this section a few different examples of generative modeling problems and the associated guarantees of the Defensive Generation algorithm. 

Rather than highlighting the strongest possible guarantees, we focus on highlighting concrete examples and the intuition behind the indistinguishability definitions. Some of these example instantiations may be of independent interest.

\paragraph{Token-Based Sequence Modeling.} To start, we consider token-based generation of sequences, such as language modeling. In this setting, the goal is given the history of tokens $x_t$ produce a conditional distribution $\mu_t$ over the next token $y_t\in \cY=[d]$ where $\cY$ is a vocabulary of $d$ tokens. 

In more detail, at each round $t=1,2,\dots$, we let $x_t$ be the entire history of tokens observed so far $x_t = (y_1, y_2, \dots, y_{t-1})$. Having seen $x_t$, the learner produces a distribution $\mu_t \in \Delta([d]) \coloneqq \{\mu \in \mathbb{R}_{\ge 0}^d: 1^\top \mu = 1\}$. That is, a
a generative model of the next token given the history,
$\mathrm{Pr}_{\mu_t}[y_t = \cdot \mid x_t]$.
Having produced $\mu_t$, the next token
$y_t \in [d]$ is revealed.

In this setting, we can produce, as a simple example, distributions $\mu_t$ that are unfalsifiable with respect to tests that linear functions of some feature embedding $\Phi(x_t)$ of the history, such as a frozen Transformer or BERT representation of the context. We define
\[
k(x,x') = \langle \Phi(x), \Phi(x') \rangle,
\qquad
\Gamma(x,x') = I_d\, k(x,x'),
\qquad
G = \sup_t k(x_t, x_t) = \sup_t \|\Phi(x_t)\|_2^2.
\]
Associated distinguishers take the form
\[
f_h(x,p,y) = h(x)^\top \vec{y}, \text{ where } h(x) = W^\top \Phi(x) \in \mathbb{R}^d \text{ and } \vec{y}=(1\{y=1\}, \dots, 1\{y=d\})
\]
with norm $\|h\|_{\cH} = \|W\|_F$. In this case, Defensive Generation guarantees that for any such $h$,
\[
\frac{1}{T}\cdot \mathrm{OIGap}_T(f_h)
=
\left|
\frac{1}{T}\sum_{t=1}^T 
\mathbb{E}\big[h(x_t)^\top (\vec{y}_t - \mu_t)\big]
\right| \le
4 \|h\|_{\cH} \sqrt{\frac{G}{T}}.
\]

This means, no matter how the tokens $y_t$ are generated, every linear function of the embedded history $\Phi(x_t)$ looks the same when evaluated $1)$ over the  empirical distribution of tokens $\sum_{t=1}^T \vec{y}_t$ or $2)$ the conditional distributions $\sum_{t=1}^t\mu_t$. We note that in this case the distinguishers $f$ are only a function of $x_t$ and the outcomes $y_t$. They ignore the side-information $p_t$, which in this case is exactly equal to the full conditional distribution $\mu_t$. Also note that the generation error bound only depends on the norms of the functions and not explicitly on the ambient dimension.

The guarantees presented here follow directly from \Cref{corr:multiclass}.

\paragraph{Predicting Correlated Rain Across the United States.}
Next, we illustrate how Defensive Generation can be used to produce conditional distributions over high-dimensional, real-valued outcomes. In particular, consider the problem of predicting how much it will rain in all 50 US states every day. Here, $y_t \in \R^{50}$ is a vector where  $y_{t,i} \in \R$ denotes the amount rainfall in state $i$ on day $t$. We normalize precipitation units so that $\|y_t\|_2 \leq 1$.

Let $x_t \in \R^{n}$ be a vector of features (e.g. atmospheric measurements, seasonality information, prior rainfall, etc.) with $\ell_2$ norm uniformly bounded by $B$. At every time $t$, the algorithm takes in $x_t$ and outputs a joint distribution $\mu_t$ over vectors in $\R^{50}$, that is rain in all 50 states.\footnote{As per \Cref{alg:moment}, this is an atomic measure supported on 51 points.} It also outputs an auxiliary statistics $p_t = (v_t,Q_t)$ where $\E_{\yt_t \sim \mu_t}[\yt_t] =v_t$ and $\E_{\yt_t \sim \mu_t}[\yt_t \yt_t^\top] =Q_t$ describing the mean and covariance of the conditional distributions of rain.

If we use the linear kernel, the algorithm guarantees (as a special case) indistinguishability with respect to all functions of the form
\[
f_{\theta}(x_t,p_t,y_t) = y_{t,i} \cdot \theta^\top x_t, \quad f_{\beta}(x_t,p_t,y_t) = y_{t,i} \cdot y_{t,j} \cdot \beta^\top x_t,\quad f_{w}(x_t,p_t,y_t) = v_t^\top w \cdot y_{t,i},
\]
where $w,\beta, \theta$ are vectors in $\R^d$, $p_t = (v_t,Q_t) = (\E_{\mu_t}[\yt_t], \E_{\mu_t}[\yt_t \yt_t^\top])$, and $i,j$ are arbitrary state indices in $\{1,\dots,50\}$. Expanding the definition of indistinguishability and plugging in the guarantees for Defensive Generation we get the following statements:
\begin{align*}
    \frac{1}{T}\cdot \oigap(f_{\theta}) &= \bigg|\frac{1}{T} \sum_{t=1}^T \theta^\top x_t (y_{t,i} -\E_{\mu_t}[\yt_{t,i}]) \bigg| \leq 4\|\theta\|_2\sqrt{\frac{(B^2 + 2)}{T}} \tag{correct mean per state},\\ 
  \frac{1}{T}\cdot \oigap(f_{\beta}) &= \bigg|\sum_{t=1}^T  \beta^\top x_t (y_{t,i}\cdot y_{t,j} - \E_{\mu_t}[\yt_{t,i}\cdot \yt_{t,j}])\bigg| \leq 4\|\beta\|_2\sqrt{\frac{(B^2 + 2)}{T}} \tag{correct covariances}, \\
  \frac{1}{T} \cdot \oigap(f_{w}) &= \bigg|\frac{1}{T}\sum_{t=1}^T  w^\top \E_{\mu_t}[\yt_t]( y_{t,i} - \E_{\mu_t}[\yt_{t,i}]) \bigg|\leq 4\|w\|_2\sqrt{\frac{(B^2 + 2)}{T}} \tag{self-consistent means}.
\end{align*}
Unpacking this a bit further, from the first set of conditions, we get that for every state $i$ the expected value of rain is uncorrelated with any linear function of the features $x$. That is, the generative model of rain has means that are conditionally correct as per this test. 

The second equation shows that not only are the means per state correct, the conditional distribution also captures the pairwise correlations between any pair of states. If a storm hits Massachusetts, it likely also hits Rhode Island. This second set of tests guarantees that the joint distribution of rain passes all these pairwise checks. It also asserts that the variances of the per state rainfall distributions are correct ($i,j$ can be the same).

The last set of conditions shows that the means are not only conditionally correct, they are self-consistent in the sense that the errors in the expected value of rain in state $i$, $\E_{\mu_t}[\yt_{t,i}] - y_{t,i}$, are uncorrelated with any linear function of the means of the distributions themselves. This is a strengthening of the indistinguishability guarantee and is not implied by either of the first two conditions.\footnote{The covariances in this example are also self-consistent, we omit the equation for the sake of concision.} See \cite{foster2006calibration} for further discussion of this point.

The guarantees presented in this example follow from \Cref{corr:Rd} (see also Appendix~\ref{app:examples}). We note how this example highlights the distinction between the conditional distributions $\mu_t$ produced by the algorithm and the statistics $p_t$ consumed by the distinguisher. While Defensive Generation produces a probability measure over points on the unit ball, the distinguishers only examine the first and second moments of the distribution. 

\paragraph{Learning Linear Dynamical Systems.} As a final example, we consider the task of finding a generative model that is indistinguishable with respect to data generated by a linear dynamical system. In particular,  consider the system
\[
z_{t+1} = A z_t + \zeta_t, 
\qquad
y_t = C z_t + \nu_t.
\]
Here, $y_t \in \R^d$ is the observation, $z_t$ is the hidden state, $A$ and $C$ are matrices, and $\zeta_t$ and $\nu_t$ are noise vectors which could be adversarial (not i.i.d). 

Assume that observations $y_t$ are uniformly bounded, $\|y_t\|\leq B$. This is true whenever the initial hidden state $x_0$ is bounded, the noise terms $(\nu_t, \zeta_t)$ are bounded, and the matrix $A$ has spectral radius strictly less than 1 ($A$ is strictly stable). To keep notation simple, we set $B=1$ (assuming strict stability one can always renormalize). Now fix a history length $\ell \ge 1$ and set the features $x_t$ in the online protocol to be the history of observations $y_j$ up to some lag $\ell$:
\[
x_t := ( y_{t-1}, y_{t-2}, \dots,  y_{t-\ell}) \in \mathbb{R}^{d\ell}.
\]
Note that $
\|x_t\|_2^2 \le \ell.$ Given $x_t$, the Defensive Generation algorithm produces a probability measure $\mu_t$ over points in $\R^d$ along with statistics $p_t =(\E_{\mu_t}[\yt_t], \E_{\mu_t}[\yt_t \yt_t^\top])$.

Again using the linear kernel $\Gamma((x,p),(x',p'))= I \cdot k(x,x')$, by \Cref{corr:Rd},the Defensive Generation algorithm guarantees indistinguishability with respect to all distinguishers of the form
\[
f(x_t,p_t,y_t)
=
\sum_{i=1}^d  x_t^\top \alpha_i \cdot y_{t,i}
+
\sum_{1 \le i \le j \le d} 
 x_t^\top \beta_{ij}  \cdot  (y_{t,i} y_{t,j}),
\]
where $\alpha_i$ and $\beta_{i,j}$ are all vectors in $\R^d$. In particular, for any such function $f$,\footnote{Compared to the distinguishers in the previous example, which looked at each mean $y_{t,i}$ individually, each $f$ in this example has a much larger scale, since it looks at all of the terms simultaneously. This is reflected by the term in parentheses in \eqref{eq:lds}, which increases when the number of dimensions $d$ increases.} 
\begin{align}\label{eq:lds}
    \oigap(f) \leq \left( \sum_{i=1}^d \|\alpha_i\|_2 + \sum_{1\leq i\leq j \leq d} \|\beta_{i,j}\|_2 \right)\sqrt{4T \ell}.
\end{align}
This in particular means that the Defensive Generation algorithm produces probability distributions $\mu_t$ that ``look like'' they generated the observations $y_t$, at least from the perspective of any linear function of the truncated history of observations,
\begin{align*}
\frac{1}{T}\! \sum_{t=1}^T \!\Bigg[  \sum_{i=1}^d  x_t^\top \alpha_i \cdot y_{t,i}
+
\sum_{\mathclap{1 \le i \le j \le d}}
 x_t^\top \beta_{ij}  \cdot  (y_{t,i} y_{t,j}) \Bigg]
\approx
\frac{1}{T}\! \sum_{t=1}^T \!\Bigg[\sum_{i=1}^d  x_t^\top \alpha_i \cdot \E_{\mu_t}[\yt_{t,i}]
+
\sum_{\mathclap{1 \le i \le j \le d}}
 x_t^\top \beta_{ij}  \cdot \E_{\mu_t} [\yt_{t,i} \yt_{t,j}]\Bigg]\!.
\end{align*}

This indistinguishability guarantee of the probability measures $\mu_t$ holds for any truncation length $\ell$ and with no stochasticity assumptions on the noise. Furthermore, the hidden state $z_t$ can be infinite-dimensional and the pair of matrices $(A,C)$ need not satisfy any observability conditions. We only required that the observations $y_t$ had bounded norm. 

Moreover, unlike other approaches to learning dynamical systems that only produce point estimates $\E_{\mu_t}[\yt_t]$ for $y_t$, our algorithm produces a full probability measure with calibrated estimates of what the covariances $\E_{\mu_t}[\yt \yt^\top]$ will be. 

Lastly, we note that our procedure does not recover the matrices $(A,C)$. It only learns to produce distributions that are indistinguishable with respect to the true data generated by the system. Using  arguments from the literature (e.g., \cite{simchowitz2021statistical}), it is plausible one can show that by setting $\ell$ in the order of $\cO(\log(T))$, the means of the distributions $\mu_t$ will be as good as those produced by a Kalman filter. However, we defer a detailed investigation to future work.

\paragraph{Generative Models of Educational Oucomes.}

As a final example, we mention how the techniques we develop in this paper can provide fine-grained indistiguishability guarantees for generative models of social outcomes, like student test scores. Our algorithm can generate conditional distributions over scalar-valued outcomes that are valid even after conditioning on rich subsets of the features and examining higher order moments. 

To illustrate this, assume that at every time step $t$ we see students with Boolean features $x_t \in \cX = \{\pm 1\}^n$ and we want to output a probability distribution $\mu_t$ over scalar outcomes $y_t \in [0,1]$ describing their test scores in an exam. Using the polynomial kernel, we can guarantee that, not just the mean, but any fixed number of moments of these distributions will be accurate over all subpopulations $c(x) \subseteq \{\pm 1\}^n$ computable by depth-$r$ decision trees.
\begin{align*}
     \frac{1}{T}\sum_{t: c(x_t)=1}^T y_t^j  \approx   \frac{1}{T}\sum_{t: c(x_t)=1}^T \E_{\mu_t} [\yt_t^j]  \quad \text{ for all } j=1,\dots 2d.
\end{align*}
Here, $c(x)$ is any Boolean function computable by a decision tree of depth $r$. The Defensive Generation algorithm guarantees that if we look at the subset of points in $\{\pm1\}^n$ where $c(x)=1$, the empirical moments of students tests scores $y_t$ will match those of the conditional distributions $\mu_t$. This guarantee holds not just for any fixed tree $c$ or moment $j$, but rather for all trees and moments $j \in [d]$ \emph{simultaneously}.


In this case, we let each distinguisher have access to the vector of moments of the distribution $\mu_t$ produced by the algorithm, $p_t = (1, \E_{\mu_t}[y_t], \dots, \E_{\mu_t}[y_t^{2d}])$. Defensive Generation is able to produce these distributions efficiently, as we show in more technical terms in \Cref{example:scalar}.

\subsection{Related Work}\label{sec:preliminaries}

The notion of outcome indistinguishability was first defined in the batch case and for binary outcomes by \citet{dwork2021outcome}. They proved that OI was equivalent to the influential notion of multicalibration \citep{hebert2018multicalibration} if the distinguishers $f$ cannot examine the underlying computational circuits that produce the samples. Using ideas on \emph{moment} multicalibration from \cite{jung2021moment}, these equivalences between OI and multicalibration (for the batch setting) were extended to the case of non-Bernoulli outcomes by \cite{beyondbernoulli}.

Following (and prior to) work on OI and multicalibration in the batch setting, there's been significant interest in extending the analyses to the online setting, with the vast majority of work focused on the case of predicting (equivalently, generating) a binary outcome. Some of this work goes back to \cite{foster1998asymptotic,sandroni2003calibration,vovk2007k29,foster2006calibration}. Following \cite{hebert2018multicalibration} there was renewed interest in the problem \citep{gupta2021OnlineML,okoroafor2025near}. Our work builds on that of \cite{kernelOI} who focused on outcome indistinguishable models of binary outcomes with respect to scalar valued RKHSs.

Our work is also closely related to work designing algorithms for high-dimensional multicalibration. In this vein, \cite{noarov2025high} design algorithms which achieve $\sqrt{T\log(|\cF|)}$ regret for this problem but require enumerating over the functions in a finite set $\cF$. Our work also builds on the idea of guaranteeing indistinguishability for richer outcome spaces and non-linear distinguishers by converting OI to a multicalibration problem in a higher-dimensional space where non-linear functions become linear in an expanded basis \citep{lu2025sample,gopalan2022loss,gupta2021OnlineML}. We build on these results to design computationally efficient algorithms that can cope with distinguisher classes $\cF$ that are infinitely large. 

Lastly, our technical approach relies and expands on the philosophy of \emph{defensive forecasting}, a methodology introduced by \cite{vovk2005defensive} for online prediction where forecasts are derived by correcting past mistakes. \cite{perdomo2025defense} provide an overview of this line of work and show how it can be used to design algorithms for online calibration, quantile regression, and loss minimization (amongst others). Our core results essentially extend the meta-algorithm discussed in \cite{perdomo2025defense} to work in high-dimensional settings. 

We defer a discussion of expected variational inequalities to \Cref{sec:calibration}.

\section{Technical Overview}
\label{sec:techniques}

Our work builds on layers of ideas from different research areas. In this section, we provide a conceptual overview of how these relate and lead to our final result.

Recall that the goal is, given features $x_t$, produce a (conditional) distribution $\mu_t$ over outcomes in a set $\cY$ that withstands falsification with respect to the true revealed outcome $y_t$ for that $x_t$. That is, $f(x_t, p_t, y_t) \approx \E_{\yt_t \sim \mu_t}[f(x_t,p_t, \yt_t)].$
This goal is well-defined even if the distinguishers only examine the features and the sampled outcome $f(x_t,p_t,y_t) =f(x_t,y_t)$. Allowing $f$ access to the ``side information'' in $p_t$ only strengthens the indistinguishability guarantee.\footnote{Using the terminology from \cite{dwork2021outcome}, we operate within the sample-access formulation of OI.}

For certain classes of outcomes $y_t$, we can guarantee indistinguishability even if we provide distinguishers with a full description of the conditional distribution $\mu_t$ from which we sample $\yt_t$. That is, we let $p_t= \mu_t$. In the binary case, this conditional distribution is just a single number $\Pr_{\mu_t}[\yt_t=1]$. And in the multiclass case where $\cY=[d]$, this is a point on the simplex $\mu_{t,j} = \Pr_\mu[\yt_t = j]$ for $ j \in [d]$. However, for continuous outcomes, one might in principle require infinitely many parameters to specify the full conditional distribution over $\yt \in [-1,1]$. 

To address this computational issue, when dealing with real-valued outcomes $y$, we restrict ourselves to guaranteeing \emph{oblivious} outcome indistinguishability as defined in \cite{beyondbernoulli}. Here, we restrict the class of distinguishers $f$ to those that can be written as a function of a finite-dimensional vector of statistics $s(y)$ living in set $\cZ \subset \R^d$, $$f(x,p,y) = g_f(s(y), x, p).$$
For instance, if the distinguisher $f$ only examines the first two moments of the distribution then we can write $f(x,p,y) = g_f(y, y^2, x,p)$ where $s(y) = (y, y^2)$ is a vector of sufficient statistics. 

In many important cases, these distinguishers are in fact linear in the sufficient statistics. That is, for every $f$, there exists a function $h_f(x,p): \cX \times \cZ \rightarrow \R^d$ such that $f(x,p,y)$  is equal to $ h_f(x,p)^\top s(y)$ for all $(x,p)$. 
This ``linearization'' is always true for the case of discrete outcomes since the conditional distribution is a sufficient statistic for any $f$. If we can write $f(x,p,y) =  h(x,p)^\top s(y)$, then outcome indistinguishability reduces to  online high-dimensional, multicalibration with respect to the collection of functions $\cH$ that depend on $\cF$. More precisely, 
\begin{align*}
  \oigap_T & =\sup_{f \in \cF} \bigg|    \sum_{t=1}^T \E_{p_t}[f(x_t, p_t, y_t)] - \sum_{t=1}^T \E_{p_t, \yt_t \sim \mu_t} [f(x_t, p_t, \yt_t)] \bigg| \\
  &= \sup_{h\in \cH}  \bigg| \sum_{t=1}^T    \E_{p_t}[h(x_t,p_t)^\top s(y_t)]  - \E_{p_t,\mu_t}[ h(x_t,p_t)^\top s(\yt_t)]  \bigg| =  \sup_{h \in \cH}  \bigg| \sum_{t=1}^T  \E_{p_t}[h(x_t,p_t)^\top (z_t - p_t)]\bigg|,
\end{align*}
where in the last equality we let $z_t = s(y_t)$ and set $p_t = \E_{\mu_t}[s(\yt_t)]$. 

Therefore, as long as we can solve for a distribution $\mu_t$ over $\cY$ such that $\E_{\mu_t}[s(\yt_t)] = p_t$, we can reduce online outcome indistinguishability to high-dimensional multicalibration. In particular, to produce the conditional distribution $\mu_t$, given $x_t$ we first produce a high-dimensional multicalibrated forecast $p_t$ such that  $p_t \approx z_t = s(y_t)$, and then solve for $\mu_t$.\footnote{Solving for $\mu_t$ is trivial in the discrete case where $p_t=\mu_t$. It is less obvious for the scalar case where $p_t = (\E_{\mu_t} [\yt], \dots, \E_{\mu_t} [\yt^{\; d}])$, but we show it can still be done by leveraging some results from semi-algebraic optimization.}

Having established this reduction, our central algorithmic contribution is a procedure that efficiently guarantees online high-dimensional multicalibration with respect to any vector-valued reproducing kernel Hilbert space $\cH$. These are rich spaces that can express functions such as all polynomials and which can be learned efficiently. We provide a brief primer for those unfamiliar.

A vector-valued RKHS is a Hilbert space of functions $\cH \subseteq \{\cX \times \cZ \rightarrow \R^d\}$ that is equal to the closure of the set $$h(x,p) = \sum_{i=1}^n \Gamma((x_i,p_i), (x,p)) \theta_i.$$
Here, $\Gamma((x,p),(x',p')) \in \R^{d \times d}$ is a matrix-valued kernel, $\theta_i$ are vectors, and $(x_i,p_i)$ are elements in $\cX \times \cZ$. In particular, $\Gamma$ is a symmetric function whose outputs are positive semidefinite matrices and which satisfies the reproducing property. For any $h \in \cH$ and $\theta \in \R^d$, 
    $$h(x,p)^\top \theta = \langle h, \Phi(x,p) \theta \rangle_{\cH}$$
where $\Phi(x,p) = \Gamma(\cdot, (x,p))$ is the evaluation functional for the RKHS. By virtue of being a Hilbert space, an RKHS has a unique inner product $\langle \cdot, \cdot \rangle_{\cH}$ and induced norm $\|h\|_{\cH} = \sqrt{\langle h, h \rangle_{\cH}}$ that serves an instance-dependent notion of complexity.

A simple case to keep in mind is when $\Phi(x,p) \in \R^{r\times d}$ is a finite-dimensional feature map and $\Gamma((x,p),(x',p')) = \Phi(x,p)^\top \Phi(x',p')$. In this case the RKHS contains the set of functions $h(x,p) = \Phi(x,p)^\top \theta$ for $\theta \in \R^r$, the inner product is equal to the standard inner product in $\R^r$ $\langle \theta, \theta' \rangle_{\cH} = \sum_{i=1}^r \theta_i \theta_i'$, and the norm of the functions $h\in \cH$ is equal to $\|h\|_{\cH} = \|h\|_2$.

To guarantee indistinguishability with respect to vector-valued RKHS, we make use of ideas from \emph{defensive forecasting}, an algorithmic template introduced by \cite{vovk2005defensive}. Defensive forecasting is a game-theoretic strategy for prediction where forecasts are derived not by prognostication, but rather by correcting past mistakes. Our main advancement here is to draw on recent breakthroughs on algorithms for solving \emph{expected variational inequalities} (EVI) to be able to efficiently do defensive forecasting in high dimensions \citep{zhang2025expected}. 

Expected variational inequalities have appeared in different areas with different names, including ``outgoing minimax problems'' \citep{foster2021forecast} or ``accuracy certificates'' \citep{nemirovski2010accuracy}. Given a desired tolerance $\epsilon > 0$, a convex, compact set $\cZ \subset \R^d$, and an operator $S : \cZ \to \mathbb{R}^d$, the goal of an EVI is to find a (finitely supported) distribution $\cD \in \Delta(\cZ)$ such that
$$
    \E_{p \sim \cD}[S(p)^\top (z - p)] \le 0
$$
for all $z \in \cZ$.
As we discuss in \Cref{sec:calibration}, these problems always admit a $\cO(\mathrm{poly}(d, 1/\epsilon))$ time algorithm via a rather clean reduction to regret minimization.\footnote{For simplicity, throughout the discussion we suppress polynomial dependence on the diameters of $\cZ$ and $S$.} While \citet{nemirovski2010accuracy} proposed algorithms with $\log(1/\epsilon)$ complexity for the limited case in which $S(\cdot)$ is a monotone operator, it was not until recently that these problems (and some generalizations) were shown to be solvable in $\cO(\mathrm{poly}(d) \log(1/\epsilon))$ time for general operators $S(\cdot)$ \citep{zhang2025expected}. This is in stark contrast with traditional variational inequality problems (``find $p \in \cZ$ such that $S(p)^\top (z - p) \le 0$ for all $z \in \cZ$''), which are computationally intractable already when $S$ is a linear function and $\cZ$ is the product of two simplices in light of standard connections with the problem of computing Nash equilibria in bimatrix games \citep{chen2009settling,daskalakis2009complexity}.

\section{Efficient High-Dimensional Multicalibration}
\label{sec:calibration}

In this section, we present an efficient algorithm for high-dimensional multicalibration with respect to general vector-valued reproducing kernel Hilbert spaces.
Throughout the section, we assume that $\cZ \subset \R^d$ is a given convex body for which an efficient (i.e., responding to queries in time polynomial in $d$ and logarithmic in the desired precision) projection, linear optimization, or separation oracle is available. These oracles are all known to be efficiently reducible to each other in time polynomial in the dimension $d$ under minimal regularity assumptions that we assume are satisfied \citep{Grotschel1993,lee2018efficient}. Furthermore, we will denote with 
\begin{equation}\label{eq:GD}
    D \coloneqq \max_{z,z' \in \cZ}\|z-z'\|_2 \quad \text{and} \quad G \coloneqq \sup \|\Gamma(\cdot,\cdot)\|_{\text{op}}
\end{equation}
the diameter of $\cZ$ and the maximum operator norm of the matrix kernel $\Gamma$, respectively. Finally, we assume that the matrix kernel $\Gamma$ can be evaluated in $\mathrm{poly}(d)$ time for any input, and likewise that the feature map $\Phi$ can be evaluated in $\mathrm{poly}(d,r)$ time for any input (if $r$ is finite).

Given $T$ samples, our algorithm guarantees order $\sqrt{T}$ regret in polynomial-time. At the heart of the construction, our algorithm relies on recent algorithmic advances for solving expected variational inequality problems efficiently  \citep{zhang2025expected}.
To appreciate why EVIs are connected to multicalibration, and what operators arise in these contexts, it is instructive to look at the general template of \emph{defensive forecasting}. 

Defensive forecasting was introduced by \cite{vovk2005defensive}, and can be thought of as a particular instantiation of the framework of Blackwell approachability. Related ideas have appeared throughout the literature, including \citet{foster1997calibrated,sandroni2003calibration,lehrer2001any,fudenberg1999easier,foster2021forecast}, and others.
We refer the reader to \cite{perdomo2025defense} for an overview of the defensive forecasting philosophy.
At all times $t$, defensive forecasting picks distributions $\cD_t$ so that the quantity
\begin{align*}
Z_T = \left\| \sum_{t=1}^T\E_{p_t \sim \cD_t} \Phi(x_t,p_t)(z_t-p_t) \right\|_\cH
\end{align*}
is guaranteed to be sublinear (as a function of $T$) no matter how Nature selects $z_t$ (after $\cD_t$ has been picked) and $x_t$ (before $\cD_t$ has been picked). This is sufficient for ensuring sublinear multicalibration error, since by the reproducing property:
\begin{align*}
    \left| \sum_{t=1}^T\E_{p_t \sim \cD_t} h(x_t,p_t)^\top (y-p_t) \right| 
    & = \left|  \langle h,  \sum_{t=1}^T\E_{p_t \sim \cD_t} \Phi(x_t,p_t)(z_t-p_t)\rangle_\cH \right| 
    \leq \|h\|_{\cH} Z_T. 
\end{align*}
Although, strictly speaking, $Z_T$ is \emph{not} a Blackwell approachability game due to the presence of the (potentially adversarial) context $x_t$ in the otherwise bilinear objective (as a function of $\cD_t$ and $z_t$), the same idea behind Blackwell's approachability algorithm yields a sublinear guarantee. In particular, observe the recursive expansion
\begin{align*}
    Z^2_t &= Z^2_{t-1} + \big\| \E_{p_t \sim \cD_t} \Phi(x_t,p_t)(z_t-p_t) \big\|_\cH ^2 \\
    &\hspace{4.0cm}+ 2\E_{p_t \sim \cD_t} \bigg[ (z_t-p_t)^\top \underbrace{\Phi(x_t,p_t)^\top \sum_{\tau=1}^{t-1} \E_{p_\tau \sim \cD_\tau} \Phi(x_\tau,p_\tau)(z_\tau-p_\tau)}_{\eqcolon~ S_t(p_t)}\bigg].
\end{align*}
It is then clear that, as long as $\cD_t$ is selected so that
\begin{equation}\label{eq:evi body}
    \E_{p_t \sim \cD_t} \Big[S_t(p_t)^\top (z-p_t)\Big] \le 0 \qquad \forall z \in \cZ,
\end{equation}
the third term can be dropped, yielding the bound
\[
    Z^2_T \le \sum_{t=1}^T \big\| \E_{p_t \sim \cD_T} \Phi(x_t,p_t)(z_t-p_t) \big\|_\cH ^2
    \quad \implies \quad
    Z_T \le \sqrt{\sum_{t=1}^T \big\| \E_{p_t \sim \cD_T} \Phi(x_t,p_t)(z_t-p_t) \big\|_\cH ^2}.
\]
This guarantees a multicalibration error growing at a $\sqrt{T}$ rate, with constants depending on $D$, $G$, and the norm of the functions in the RKHS.
We summarize the procedure in Algorithm~\ref{alg:df}, incorporating the possibility of error $\epsilon_t$ on the right of \eqref{eq:evi body}. From the above discussion, we immediately derive the following guarantee.

\begin{algorithm}[t!]
\caption{Efficient High-Dimensional Multicalibration via Expected VIs
}
\begin{algorithmic}[1]
\State Matrix-valued kernel function $\Gamma: (\cX \times \cZ)^2 \rightarrow \R^{d \times d}$
\State \textbf{For }{$t=1,2, \dots, T$:}
\State \quad See the features $x_t$ and define $S_t: \cZ \rightarrow \R^d$ as
    \begin{align*}
    \quad S_t(p) &= \Phi(x_t,p)^\top \sum_{\tau=1}^{t-1} \E_{p_\tau \sim \cD_\tau} \Phi(x_\tau,p_\tau)(z_\tau-p_\tau) = \!\sum_{\tau=1}^{t-1} \E_{p_\tau \sim \cD_\tau} \Gamma((x_t, p), (x_\tau,p_\tau))(z_\tau-p_\tau)
    \end{align*}
\State \quad Predict $p_t \sim \cD_t$, where the distribution $\cD_t \in \Delta(\cZ)$ satisfies the expected variational inequality\!\!\!\!\! 
    \begin{align}\tag{EVI}\label{eq:evi}
\E_{p \sim \cD_t} \Big[ S_t(p)^\top (z - p) \Big] \leq \epsilon_t \qquad \forall z \in \cZ
\end{align}
    ~\hspace{5mm}for some small error $\epsilon_t$; for example, $\epsilon_t = 1$ or $\epsilon_t = 1/\sqrt t$
\end{algorithmic}
\label{alg:df}
\end{algorithm}

\begin{proposition}\label{prop:multicalibration error}
Let $\Gamma$ be a matrix-valued kernel with associated RKHS $\cH \subseteq \{\cX \times \cZ \rightarrow \R^d\}$. If the predictions $p_t\sim \cD_t$ are made according to Algorithm \ref{alg:df}, then for all $T \geq 1$ and $h\in \cH$:
\begin{align*}
     \left| \sum_{t=1}^T\E_{p_t \sim \cD_t} h(x_t,p_t)^\top (z_t-p_t) \right| \leq \|h\|_{\cH}\sqrt{T D^2 G + \sum_{t=1}^T \epsilon_t}.
\end{align*}
Hence, as long as $\sum_{t=1}^T \epsilon_t = \cO(T)$, Algorithm~\ref{alg:df} guarantees order $\cO(\sqrt{T})$ multicalibration error.
\end{proposition}

\paragraph{Efficiently Solving the EVI.} 

The key (and effectively, only) step in \cref{alg:df} is constructing a distribution $\cD_t$ over $\cZ$ that solves the expected variational inequality problem~\eqref{eq:evi}. As mentioned in \cref{sec:preliminaries}, a solution of an EVI can always be computed efficiently for any general, even discontinuous, bounded VI operator (in our case, $S_t$), given oracle access---for example, a membership, linear optimization, or separation oracle---to the convex compact domain (in our case, $\cZ$). In the rest of this section, we sketch two known general approaches for solving EVIs, with runtimes scaling polynomially and polylogarithmically in the desired inverse precision $1/\epsilon_t$, respectively.

To start, a distribution $\cD_t$ such that \eqref{eq:evi} holds can be found in $\mathrm{poly}(d,D,G,1/\epsilon_t)$ time by no-regret algorithms (see also \citet{zhang2025expected,sigecom}). In particular, consider setting up any no-regret learning algorithm outputting at every time $k$ a point $p_k \in \cZ$, and receiving as utilities (losses) the functions $p \mapsto S_t(p_k)^\top p$.\footnote{We use $k$ to indicate the iteration of the no-regret algorithm, to avoid confusion with scalar kernel $k$ used in \Cref{sec:OI}. An EVI \eqref{eq:evi} must be solved at every time $t$ in Algorithm~\ref{alg:df}, and we are proposing an iterative algorithm indexed over iterations $c$ to solve such an EVI at that time $k$.} Then, expanding the definition of what it means to have no-regret over the course of an arbitrary number $K$ of iterations, one has
\[
    \frac{\mathrm{Regret}_K}{K} = \frac{1}{K} \sum_{k=1}^K S_t(p_k)^\top(z - p_k)  = o(1) \qquad \forall z \in \cZ,
\]
implying that the uniform distribution $\cD_t$ over the set of iterates $\{p_1, \dots, p_K\}$ produced by the no-regret algorithm converges to an arbitrarily good solution of the expected variational inequality as the number of iterations $K$ grows. 
By using known regret bounds for online projected gradient ascent, setting the learning rate appropriately, and accounting for the fact that the diameter of the image of $S_t$ is of order $tGD$, we obtain the following rate (see Appendix~\ref{sec:mc reg} for the derivation).
\begin{proposition}\label{prop:mc reg bnd}
    There exists a no-regret-based implementation of \cref{alg:df} that achieves \emph{average} multicalibration error
    $
        | \frac{1}{T} \sum_{t=1}^T \E_{p_t\sim\cD_t} h(x_t,p_t)^\top (z_t - p_t) | \le \epsilon
    $
    in time $\cO(\mathrm{poly}(d, D, G) \cdot \min \{r, \epsilon^{-6}\} \cdot \epsilon^{-6})$.
    In particular, the runtime is never worse than $\cO(\mathrm{poly}(d, D, G) \cdot \epsilon^{-12}).$
\end{proposition}

Significantly faster approaches solving EVIs (and even harder generalizations) at $\mathrm{log}(1/\epsilon)$ rates have been very recently developed by leveraging a new constructive version of the minimax theorem \citep{farina2024polynomial,zhang2025expected,sigecom}. These methods are based on the ellipsoid method, and can solve EVIs to error $\epsilon_t$ with only order $\log(1/\epsilon_t)$ evaluations of the operator $S_t$. We include a high-level, self-contained description of these more advanced algorithms in Appendix~\ref{sec:evi description}. Here, we only state the final result; a proof is deferred to Appendix~\ref{sec:mc ellipsoid}.

\begin{proposition}\label{prop:mc ellipsoid bnd}
    There exists an ellipsoid-based implementation of \cref{alg:df} that achieves \emph{average} multicalibration error
    $
        | \frac{1}{T} \sum_{t=1}^T \E_{p_t\sim\cD_t} h(x_t,p_t)^\top (z_t - p_t) | \le \epsilon
    $
    in time $\tilde\cO(\mathrm{poly}(d, D, G) \cdot \min \{r, \epsilon^{-2}\} \cdot \epsilon^{-2})$.
    In particular, the runtime is never worse than $\tilde\cO(\mathrm{poly}(d, D, G) \cdot \epsilon^{-4})$, and it is of order $\tilde\cO(\mathrm{poly}(d, D, G) \cdot \epsilon^{-2})$ if the number of features $r$ is polynomial in $d, D,$ and $G$.
\end{proposition}

\section{Online Outcome Indistinguishable Generative Models}

\label{sec:OI}
\begin{algorithm}[t!]
\caption{The Defensive Generation Algorithm}
\begin{algorithmic}[1]
\State \textbf{Input:} matrix-valued kernel $\Gamma$, outcome and prediction sets $\cY$ and $\cZ$, function $s:\cY \rightarrow \R^d$
\State \textbf{For }{$t=1, \dots, T$}\textbf{:}
\State \quad Receive features $x_t$
\State \quad Produce predicted statistic $p_t \in \cZ$ via defensive forecasting (\Cref{alg:df}) to error $\eps_t$
\State \quad Solve for a  distribution $\mu_t$ such that $\E_{\yt_\sim \mu_t}[s(\yt_t)]=p_t$, output the measure $\mu_t$
\State \quad Record true outcome, $y_t \in \cY$ and statistic $s(y_t) = z_t$

\end{algorithmic}
\label{alg:OI}
\end{algorithm}

In this section, we illustrate how to use the high-dimensional multicalibration algorithm from \Cref{sec:calibration} to produce online outcome indistinguishability generative models for multiclass, scalar-valued, and vector-valued outcomes. We assume throughout that  $\eps_t = GD^2$ in \Cref{alg:OI}.

\paragraph{Algorithm Description.} As discussed in \Cref{sec:techniques}, the $\OIalg$ algorithm guarantees indistinguishability with respect to functions $f$ that can be written as $f(x,p,y)=h(x,p)^\top s(y)$ where $h$ belongs to vector-valued RKHS $\cH$. The procedure takes as input the outcome and prediction sets $\cY$ and $\cZ \subset \R^d$, as well as the function $s: \cY \rightarrow \cZ$ that together define the generative modeling task. It also takes as input the matrix-valued kernel function $\Gamma: (\cX \times \cZ)^2 \rightarrow \R^{d \times d}$ that implicitly defines the RKHS $\cH$.

$\OIalg$ is just a simple wrapper around the online multicalibration routine (\Cref{alg:df}). Given $x_t$, it appeals to the defensive forecasting procedure to find $p_t$ that is a multicalibrated forecast of $z_t = s(y_t)$. Then, it solves for a distribution $\mu_t$ such that $\E_{\mu_t}[s(\yt_t)] =p_t$. To specialize it to different outcome spaces $\cY$, we need to ensure that we can 1) implement efficient oracle access (e.g., separation) for the set $\cZ$ as discussed in \Cref{sec:calibration} and 2) be able to solve for $\mu_t$. The rest of this section shows how to do both of these things for different choices of $\cZ$ and $\cY$. 

Before discussing these adaptations, we state the end-to-end guarantee of the algorithm:
\begin{theorem}
\label{theorem:OI}
Fix an outcome space $\cY$, a function $s: \cY \rightarrow \cZ$, and let $\cF$ be a class of distinguishers $f$ that can be written as $f(x,p,y) = h_f(x,p)^\top s(y)$ for a function $h_f$ in vector-valued RKHS $\cH$. 

Given access to a sampling oracle that on input $z$ returns a probability measure $\mu$ such that $\E_{\yt \sim \mu}[s(\yt)] = z$, and a separation oracle for the set $\cZ$,\footnote{Given $z'$, a separation oracle for the set $\cZ$ returns true if $z'\in \cZ$. If $z'\notin \cZ$, it returns a hyperplane $(w,b)$ such that $w^\top z' > b$ but $w^\top z < b$ for all $z \in \cZ$ } the $\OIalg$ algorithm with $\epsilon_t = D^2 G$, where $D$ and $G$ are as in \eqref{eq:GD}, guarantees online outcome indistinguishability with respect to $\cF$. More precisely, given any $f \in \cF$, 
$$
    \oigap_T(f) \le \|h_f\|_{\cH} \sqrt{2 D^2 G T}.$$
\end{theorem}

\subsection{Generative Models for Multiclass Outcomes}
In this subsection, we specialize this main theorem to the case where $\cY$ consists of $d$ labels. In this setting, our results simplify substantially and we can guarantee indistinguishability with respect to distinguishers that examine the entire conditional distribution from which we are sampling. That is, we let $p = \mu$ where $\mu$ is a point on the simplex, $$\mu \in \cZ = \Delta^d = \{p: p_j \geq 0, \sum_{j=1}^d p_j =1\} \text{ and } \yt \sim \mu.$$

In more detail, if $\cY = \{1, \dots, d\}$, then for any function $f(x,p,y)$, we can write  
    $$f(x,p, y)= \sum_{j=1}^d f(x,p,j) \cdot 1\{y = j \}.$$
Therefore, any distinguisher $f$ can be expressed as a linear function of its values at $d$ points, $f(x,p,y) = h_f(x,p)^\top s(y)$ where  $h_f(x,p) = (f(x,p,1), \dots, f(x,p,d))$ and $s(y) = (1\{y=1\}, \dots, 1\{y=d\})\in \cZ$ is a one-hot encoded version of the outcomes $y$. 

Moreover, for this multiclass case, there is a simple formula we can use to construct function spaces that contain these functions $h_f$. In particular, assuming that there is a \emph{scalar}-valued RKHS that contains each of the scalar-valued functions $f(x,p,i)$ for $i\in[d]$, we can always construct a \emph{vector}-valued RKHS that contains the function $h_f(x,p)$ from above. 
\begin{fact}[\cite{alvarez2012kernels}]
\label{fact:rkhs}
Let $k((x,p),(x',p')) \in \R$ be a scalar-valued kernel with  RKHS $\cH_{\mathrm{scalar}} \subset \{\cX \times \cZ \rightarrow \R\} $. Then, $\Gamma((x,p),(x',p')) = I_d \cdot k((x,p),(x',p')) \in \R^{d \times d}$ is a matrix-valued kernel whose corresponding vector-valued RKHS $\cH_d$ consists of all functions of the form:
\begin{align*}
    h(x,p) = (h_1(x,p), \dots, h_d(x,p))^\top \text{ where } h_j(x,p) \in \cH_{\mathrm{scalar}} \text { for all } i \in [d].
\end{align*}
Furthermore, for any function $h \in \cH_{d}$ its RKHS norm is equal to $\|h\|_{\cH_d}^2 = \sum_{j=1}^d \|h_j\|_{\cH_{\mathrm{scalar}}}^2$.
\end{fact}
Here, there is no need to implement a sampling oracle, since the algorithm directly solves for the distribution $p=\mu$ as part of the multicalibration subroutine. Furthermore, it is simple to check whether any point is in the simplex $\cZ = \{p: p_i \geq 0, \sum_{i=1}^d p_i =1\}$, and one can easily implement a separation oracle. Tying these facts together with \Cref{theorem:OI}, we get the following corollary:\footnote{We use $\mu_t$ and $p_t$ interchangeably in \Cref{corr:multiclass} since they are the same for this multiclass setting.}

\begin{corollary}
\label{corr:multiclass}
Set $\cY = [d]$ and define $\cZ$ to be the simplex in $\R^d$, $\cZ =\{p: p_j \geq 0, \sum_{j=1}^d p_j =1\}$. 
Let $k$ be a kernel such that $\sup_{x,p}|k((x,p),(x,p))|\leq G$ with an associated scalar-valued RKHS $\cH_{\mathrm{scalar}}$, such that $f(x,p, j) = h_j(x,p) \in \cH_{\mathrm{scalar}}  \text{ for all } j \in [d].$

Then, Defensive Generation with kernel $\Gamma = I_d \cdot k$ produces distributions $\mu_t$ over outcomes $\cY$ that are online outcome indistinguishable with respect to the set of functions $f(x,\mu,y)$ where $h_j(x,\mu) = f(x,\mu,j)$ is in $\cH_{\mathrm{scalar}}$ for every $j \in [d]$. In particular,  for every such $f$, $$\oigap_T(f) \le 4 \left( \sum_{j=1}^d \|h_j\|_{\cH_{\mathrm{scalar}}} \right) \sqrt{TG}.$$
\end{corollary}
We can further specialize this statement to guarantee indistinguishability with respect to common function classes as per this example:
\begin{example}
\label{example:multiclass}
Let $\cY = [d]$, $\cX = \{x: \|x\|_2\leq 1\} \subset \R^r$, and set $k((x,p), (x',p')) = \langle x, x' \rangle  + \langle p,p'\rangle$. 
Running the $\OIalg$ with $\Gamma((x,p),(x',p')) = I_d \cdot k((x,p), (x',p'))$ guarantees online outcome indistinguishability with respect to the set of functions,
\begin{align*}
    \cF = \{f(x,p,y): f(x,p,j) = \langle x, \theta_j \rangle + \langle p, \theta'_j \rangle, \quad (\theta_i, \theta_i') \in \R^r \times \R^d \text{ for all } j \in [d] \}.
\end{align*}
In particular, if we restrict $\|\theta_j\|_2 + \|\theta_j'\|_2 \leq B$ for every $j \in [d]$, then we get that $\oigap_T(f)$ is at most $4 dB \sqrt{T}$ for any $f$.
Furthermore, at time $t$ the algorithm runs in time $\cO(\poly(d, B) \log(t))$.
\end{example}

\subsection{Generative Models for Scalar-Valued Outcomes}

As a second application, we show how Defensive Generation can be used to produce distributions $\mu_t$ over scalar outcomes $y_t$ that are outcome indistinguishable from the perspective of any $f$ that examines higher-order moments of the distribution, up to some fixed degree $2d$. In symbols, we focus on the setting in which $\cY = [-1,1]$ and $s(y) = (1, y, y^2, \dots, y^{2d})$ and $\cZ$ is the (convex) set of moments:
\begin{align}
\label{eq:moment_cone}
    \mathcal{Z} \coloneqq \left\{\left(\int_{[-1,1]} y^0 \di \mu, \dots, \int_{[-1,1]} y^{2d} \di\mu\right): \mu \text{ is a Borel probability measure on }[-1,1]\right\}.
\end{align}
To extend our results to this setting, we leverage techniques from semi-algebraic optimization.
As we discuss in \Cref{app:moment cone}, $\mathcal{Z}$ admits an efficient separation oracle, thus satisfying the preconditions of our construction in \Cref{sec:calibration} for efficiently producing defensive forecasts $z_t \in\mathcal{Z}$. 
Furthermore, once a moment vector $z_t \in \mathcal{Z}$ is selected, it is possible to efficiently construct a discrete distribution $\mu_t$ on $[-1,1]$ whose moments match those specified by $z_t$. By sampling a $\yt_t \sim \mu_t$, we can complete the roundtrip and yield predictions that are indistinguishable with respect to tests that only examine a fixed number of moments. Tying these ideas together, we get the following corollary:

\begin{corollary}
\label{corr:scalar}
Set $\cY = [-1,1]$ and define $\cZ \subset \R^{2d +1}$ to be the set of moments as in \Cref{eq:moment_cone}. Let $\Gamma$ be a matrix-valued kernel with associated vector-valued RKHS $\cH \subseteq \{\cX \times \cZ \rightarrow \R^d \}$ and let $\cF$ be the set of functions, 
\begin{align*}
    \cF \coloneqq \{f(x,p,y): f(x,p,y)= h(x,p)^\top s(y) \text{ where } s(y) = (1, y, \dots, y^{2d})\}.
\end{align*}

The $\OIalg$ algorithm with the sampling and separation oracles from \Cref{app:moment cone} produces distributions $\mu_t$ and statistics $p_t^\top = (1, \E_{\mu_t}[\yt_t], \dots, \E_{\mu_t}[\yt_t^{\;2d}])$ such that for every $f \in \cF$, 
\begin{align*}
\oigap_T(f) \leq \|h\|_{\cH} \sqrt{\sum_{t=1}^T \E_{p_t} (p_t - s(y_t))^\top \Gamma((x_t,p_t),(x_t,p_t))(p_t - s(y_t)) + GD^2T}.
\end{align*}
\end{corollary}
Using this we can generate distributions over scalar outcomes that fool all low degree tests. That is, we can use to Defensive Generation algorithm to produce distributions $\mu_t$ over scalar outcomes in $[-1,1]$ such that all low degree moments are conditionally valid over all subsets of the Boolean hypercube computable by low-depth decision trees.
\begin{example}
\label{example:scalar}
Let $\cX =\{\pm 1\}^n$ be the Boolean hypercube and consider the polynomial kernel, $$k_{\poly}(x,x') = (1+ x^\top x')^r.$$ It is a well-known result that the RKHS $\cH_{\poly}$ for this kernel contains all Boolean functions $c: \{\pm 1\}^n \rightarrow \{0,1\}$ computable by decision trees of depth at most $r$.\footnote{See Proposition 4.12 in \cite{kernelOI} for a proof. As per \citet{o2014analysis},  a decision tree is a representation of a boolean function as a rooted binary tree in which the internal nodes are labeled by coordinates $[i] \in [n]$, the outgoing edges are labeled by $-1$ and $1$ and the leaves have real numbers corresponding to the outputs.}

Moreover, by \Cref{fact:rkhs}, the matrix-valued kernel $\Gamma((x,p),(x,p)) = I_d \cdot k_\poly(x,x')$ contains all functions, $h(x) = (c_1(x), \dots, c_d(x))$ where each $c_i$ is a Boolean function in $\cH_{\poly}$. Consequently, if we consider the space of functions $\cF$ that can be written as, 
\begin{align}
\label{eq:poly_kernel}
    f(x,p,y) = h(x)^\top s(y) = \sum_{j=1}^d c_j(x) y^j,
\end{align}
this space contains all tests of the form $f(x,p,y)= 1\{c(x)=1\} \cdot y^j$. Therefore, the Defensive Generation algorithm guarantees that for any $j=1,\dots, 2d$ and all low-degree boolean functions $c$,
\begin{align}
\label{eq:all_powers}
   \lim_{T\rightarrow\infty} \bigg| \frac{1}{T} \sum_{t=1}^T 1\{c(x_t) =1\} y_t^j  - \frac{1}{T} \sum_{t=1}^T 1\{c(x_t) =1\} \E_{\yt_t \sim \mu_t} \E_{\mu_t}[\yt_t^j] \bigg| = 0.
\end{align}
That is, the first $2d$ moments of the conditional distributions we produce (online) match those of the revealed outcomes. Furthermore, this is true conditionally over all the subsequences $\{x_t: c(x_t) = 1,  c \text{ is a decision tree of depth } r\}$. 
More formally, for all functions $f$ of the form in \Cref{eq:poly_kernel}, 
$$\oigap_T(f) \leq \cO(\poly(d) \cdot n^r \cdot \sqrt{T}).$$
\end{example}

While \cite{kernelOI} achieved the guarantee in \Cref{eq:all_powers} for the case where $j=1$, our algorithm produces distributions which (simultaneously) satisfy the guarantee for all $j=1, \dots, d$. This result is also closely related to the guarantees in \cite{gupta2021OnlineML}. Their results hold for higher order moments, but only with respect to a finite class of distinguishers. 

We note that one interesting consequence of producing distributions that have these high order guarantees of validity is that one can use them to derive high-probability prediction intervals via Chebyshev's inequality. We refer the reader to \cite{gupta2021OnlineML} for further details.

\subsection{Generative Models for High-Dimensional Outcomes}

As a final application, we consider guaranteeing outcome indistinguishability when outcomes are high-dimensional, $\cY \subset \R^d$.

\paragraph{Indistinguishablity of conditional expectations}
To start, we show that we can always guarantee online outcome indistinguishability when the distinguishers $f$ are linear functions of the first moment of $y$, $f(x,p,y) = h(x,p)^\top y$, and $s(y)=y$. That is, they guarantee that the distributions produced by the learner have first moments that are conditionally correct as measured by the functions $h$. 

For this setting, we let $\cZ = \cY \subset \R^d$ which is assumed to be any compact, convex set with diameter bounded by $D$ with an efficient separation oracle (e.g. $\cY=[-1,1]^d$). Since $s(y)$ is just the identity function, the problem of finding a distribution $\mu_t$ such that $\E_{\yt \sim \mu_t}[s(\yt_t)] = \E_{\yt \sim \mu_t}[\yt_t]=p_t$ is easy, just return a point mass distribution at $p_t$. 
Here, we have the following corollary:

\begin{corollary}
Let $\Gamma$ be a matrix valued kernel with operator norm uniformly bounded by $G$ and  with vector-valued RKHS $\cH \subseteq \{\cX \times \cZ \rightarrow \R^d\}$ where $\cZ$ has diameter bounded by $D$. Then, the $\OIalg$ algorithm with kernel $\Gamma$ guarantees online OI with respect to the set of functions $f(x,p,y) = h(x,p)^\top y $ where $h(x,p)\in \cH$. For any $f \in \cF$, $\oigap_T(f) \leq \|h\|_{\cH}  \sqrt{2T D^2G}.$
\end{corollary}
As a concrete instantiation of this result, we can guarantee indistinguishability with respect to the set of functions:
    $\cF = \{f(x,p,y) =  h(x,p)^\top y  \text{ where } h(x,p)= \Phi(x,p)\theta \text{ for } \theta \in \R^r\}.$
    
Here, $\Phi$ is a finite-dimensional feature map $\Phi(x,p) \in \R^{d\times r}$ and $h(x,p)$ is any function that lies in the span of these features. For instance, if we let, 
\begin{align}
\label{eq:finite_g}
    \Phi(x,p)^\top = \big[
        g_1(x,p) \mid \dots \mid g_r(x,p) 
    \big] \in \R^{d \times r}
\end{align}
where $g_1,\dots,g_r$ is any arbitrary set of $r$ functions then we get that $\cH = \mathrm{span}(\{g_1, \dots, g_r\})$. This is just one choice of an explicit feature map, but one could of course consider others.

\begin{example}
\label{example:vector}
Let $\cY=\cZ$ be a compact subset of $\R^d$ with diameter $D$ and let $\Phi: \cX \times \cZ \rightarrow \R^{r \times d}$ be a feature map. The $\OIalg$ algorithm with matrix valued kernel $\Gamma((x,p),(x',p')) = \Phi(x,p)^\top \Phi(x',p')$ guarantees online outcome indistinguishability with respect to the set of functions $f(x,p,y)= h(x,p)^\top y $ where $h(x,p) = \Phi(x,p) \theta$. In particular, for every such $h$,
\begin{align*}
\oigap_T(f) \leq \|\theta\|_{2}  \sqrt{2D^2 T \cdot \sup_{x,p} \|\Phi(x',p')\|^2_{\mathrm{op}}}.
\end{align*}
Therefore, if we set $\Phi$ as in \eqref{eq:finite_g} for functions $g_1,\dots,g_r$, where $\max_{1\leq j \leq r} \|g_j(x,p)\|_2 \leq B$, then for any $f(x,p,y) = g_j(x,p)^\top y$, we have $\oigap_T(f) \leq  \sqrt{2rB^2D^2T}$.
\end{example}

\paragraph{Indistinguishability with respect to higher-order moments}
In general, the defensive generation approach described so far can guarantee efficient online outcome indistinguishability with respect to $d$-order moments, when the set of such moments admits an efficient separation oracle. This condition is equivalent to the existence of efficient algorithms for the minimization of polynomials of degree up to $d$ on $\cY$.

As an illustration, we show that on the high-dimensional ball $\cY = \{y \in \R^d: \|y\|_2 \leq 1\}$, we can guarantee indistinguishability with respect to functions $f$ that examine the first two moments of the distributions $\mu_t$, 
\begin{align}
\label{eq:f_vec}
f(x,p,y) = h(x,p)^\top s(y) \text{, and }s(y)=(y_1, \dots, y_d\;, \;y_1^2\;, y_1 y_2 \;, \dots, y_i y_j) \text{ for } i,j \in [d], i\leq j.
\end{align}
In other words, they guarantee that the distributions $\mu_t$ produced online have high-dimensional means, $\E_{\mu_t}[\yt_t]$, and covariances, $\E_{\mu_t}[\yt_t \yt_t^\top] \in \R^{d \times d}$, that are conditionally correct as measured by $h$.
In this setting, the set of moments $\cZ$ is 
\begin{align}
\label{eq:2_cone}
    \cZ = \{(v,Q): \exists \; \text{a probability measure } \mu \text{ over } \|y\|_2 \leq 1 \text{ such that } \E_{\mu}[y]=v, \E_{\mu}[yy^\top] = Q\}
\end{align}
Since quadratic functions on the unit ball can always be optimized efficiently (by performing a singular value decomposition), a separation oracle for $\cZ$ can always be constructed efficiently. In fact, in this case we do not even need to resort to complicated separation oracle constructions: as it turns out, the set $\cZ$ admits a simple characterization. Indeed,
an application of the S-Lemma \citep{yakubovich1971s} lets us write
\begin{align}\label{eq:S}
    \cZ = \{(v,Q): Q \succeq 0, Q \succeq vv^\top, \Tr[Q]\leq 1\},
\end{align}
where $\Tr[Q]$ denotes the trace of $Q$ and $A\succeq B$ denotes that $A - B$ is positive semidefinite. (A more constructive proof of \eqref{eq:S} via some intermediate results that will be used later in this section is available in Appendix \ref{app:2_cone}.) From this rewriting, it is straightforward to efficiently check whether any given proposed moments $(v',Q')$ belong to $\cZ$, or return a separating hyperplane (violated constraint) otherwise. 

Moreover, given $(v,Q)\in \cZ$, we can efficiently find an atomic probability measure $\mu$ over the unit ball such that $\E_{y \sim \mu}[y]=v$ and $\E_{\mu}[yy^\top]=Q$. It involves nothing more complicated than taking the SVD of a matrix. We present the procedure in \Cref{alg:moment} and the proof of its correctness in Appendix \ref{app:2_cone}. 

\begin{algorithm}[t!]
    \caption{Backfitting a probability measure $\mu$.}
    \label{alg:moment}
    \begin{algorithmic}[1]
        \Statex \textcolor{gray}{\texttt{//} We use $\delta(y)$ to denote a point mass on a point $y$.}
        \State \textbf{Input:} PSD matrix $Q$ and vector $v\in \mathbb{R}^d$ such that $Q \succeq vv^\top$
        \State Compute the SVD of $Q - vv^\top = \sum_{i=1}^d \sigma_i u_i u_i^\top$
        \State \textbf{For } $i = 1 \dots d$,
        \State \quad Solve for the 2 real-valued roots $t_i^+> 0$ and $t_i^-<0$ that satisfy $\|v + t \cdot u_i\|=1$
        \State \quad Define $$y_i^+ = v + t_i^+u_i, \;y_i^- = v + t_i^-u_i \
        \quad \text{ and }\quad \lambda_i^+ =\frac{\sigma_i}{t_i^+(t_i^+- t_i^-)}, 
        \lambda_i^- =\frac{\sigma_i}{t_i^-(t_i^-- t_i^+)}$$
        \State \textbf{Return} $\mu = \lambda_0 \delta(v) + \sum_{i=1}^d[\lambda_i^+ \delta(y_i^+) +  \lambda_i^- \delta(y_i^-)]$ where $\lambda_0 = 1-\sum_{i=1}^d (\lambda^+_i + \lambda^-_i)$
    \end{algorithmic}
\end{algorithm}

Given that we can implement a separation oracle and a way to solve for a probability measure $\mu$ whose moments match a target vector of moments in $\cZ$, we have the following corollary:

\begin{corollary}
\label{corr:Rd}
Let $\Gamma$ be a matrix-valued kernel with operator norm uniformly bounded by $G$ and  with vector-valued RKHS $\cH \subseteq \{\cX \times \cZ \rightarrow \R^{d + d(d+1)/2}\}$ where $\cZ$ is as in \eqref{eq:2_cone} and $s(y) \in \R^{d + d(d+1)/2} $. Then, $\cZ$ has diameter at most $\sqrt{8}$ and the $\OIalg$ algorithm with kernel $\Gamma$ guarantees online OI with respect to the set of functions $f(x,p,y)$ defined in \eqref{eq:f_vec}. For any $f \in \cF$, $$\oigap_T(f) \leq 4\|h\|_{\cH}  \sqrt{TG}.$$
\end{corollary}
As a concrete instantiation of this result, we can guarantee indistinguishability with respect to the set of functions in \Cref{eq:f_vec}
where $s(y) \in \R^{m}$ for $m = d + d(d+1)/2$ is a vector containing all the variables in the first and second moments of $y$ (e.g. $y_i,y_i^2$ and also the mixed pairs $y_iy_j$ for $y=(y_1, \dots  y_d)$). In this example, $\Phi$ is a finite-dimensional feature map $\Phi(x,p) \in \R^{m \times r}$ and $h(x,p)$ is any function that lies in the span of these features. For instance, if we let, 
\begin{align}
\label{eq:finite_g}
    \Phi(x,p)^\top = \big[
        g_1(x,p) \mid \dots \mid g_r(x,p) 
    \big] \in \R^{m \times r}
\end{align}
where $g_1,\dots,g_r$ is any arbitrary set of $r$ functions then we get that $\cH = \mathrm{span}(\{g_1, \dots, g_r\})$. This is just one choice of an explicit feature map, but one could of course consider others.

\begin{example}
\label{example:vector2}
Let $\cZ$ (see \Cref{eq:2_cone}) be the set of first and second moments of probability distributions over the unit ball  and let $\Phi: \cX \times \cZ \rightarrow \R^{r \times m}$ be a feature map. 

The $\OIalg$ algorithm with  kernel $\Gamma((x,p),(x',p')) = \Phi(x,p)^\top \Phi(x',p')$ guarantees online outcome indistinguishability with respect to the set of functions $f(x,p,y)=h(x,p)^\top s(y)$  where $h(x,p) = \Phi(x,p) \theta$. In particular, for every such $h$,
\begin{align*}
\oigap_T(f) \leq 4 \|\theta\|_{2}  \sqrt{ T \cdot \sup_{x,p} \|\Phi(x',p')\|^2_{\mathrm{op}}}.
\end{align*}
Therefore, if we set $\Phi$ as in \eqref{eq:finite_g} for functions $g_1,\dots,g_r$, where $\max_{1\leq j \leq r} \|g_j(x,p)\|_2 \leq B$, then for any $f(x,p,y) = g_j(x,p)^\top s(y)$, 
$$\oigap_T(f) \leq  4\sqrt{rB^2T}.$$
\end{example}

We note that while we presented the results in this last section for the case where $\cY$ is the unit ball. They apply more generally to the case where $\cY$ is any ellipsoid $$\cY = \{x: (x-b)^\top A (x-b) \leq 1, A\succ 0, b\in R^d\},$$
by simply changing basis $x \mapsto A^{-1/2}x +b$. 

\section{Acknowledgments}

GF was supported in part by the National Science Foundation award CCF-2443068, the Office of Naval Research grant N000142512296, and an AI2050 Early Career Fellowship. We would like to thank Benjamin Recht for his helpful comments.

{\small
\bibliography{ref}

@article{perdomo2025defense,
  title={In Defense of Defensive Forecasting},
  author={Perdomo, Juan Carlos and Recht, Benjamin},
  journal={arXiv preprint arXiv:2506.11848},
  year={2025}
}

@inproceedings{zinkevich2003online,
	author = {Zinkevich, Martin},
	booktitle = {International Conference on Machine Learning},
	title = {Online convex programming and generalized infinitesimal gradient ascent},
	year = {2003}}

@article{gupta2021OnlineML,
  title={Online Multivalid Learning: Means, Moments, and Prediction Intervals},
  author={Varun Gupta and Christopher Jung and Georgy Noarov and Mallesh M. Pai and Aaron Roth},
  journal={Innovations in Theoretical Computer Science},
  year={2022},
}

@inproceedings{jung2021moment,
  title={Moment multicalibration for uncertainty estimation},
  author={Jung, Christopher and Lee, Changhwa and Pai, Mallesh and Roth, Aaron and Vohra, Rakesh},
  booktitle={Conference on Learning Theory},
  year={2021},
}

@article{foster1997calibrated,
  title={Calibrated learning and correlated equilibrium},
  author={Foster, Dean P and Vohra, Rakesh V},
  journal={Games and Economic Behavior},
  year={1997},
}

@inproceedings{foster2006calibration,
  title={Calibration via regression},
  author={Foster, Dean P and Kakade, Sham M},
  booktitle={IEEE Information Theory Workshop},
  year={2006},
}

@article{foster1998asymptotic,
  title={Asymptotic calibration},
  author={Foster, Dean P and Vohra, Rakesh V},
  journal={Biometrika},
  year={1998},
}

@article{fudenberg1999easier,
  title={An easier way to calibrate},
  author={Fudenberg, Drew and Levine, David K},
  journal={Games and Economic Behavior},
  year={1999},
}

@inproceedings{dwork2021outcome,
	author = {Dwork, Cynthia and Kim, Michael P and Reingold, Omer and Rothblum, Guy N and Yona, Gal},
	booktitle = {ACM Symposium on Theory of Computing},
	title = {Outcome Indistinguishability},
	year = {2021}}

@article{zhang2025expected,
  title={Expected variational inequalities},
  author={Zhang, Brian Hu and Anagnostides, Ioannis and Tewolde, Emanuel and Berker, Ratip Emin and Farina, Gabriele and Conitzer, Vincent and Sandholm, Tuomas},
  journal={International Conference on Machine Learning},
  year={2025}
}

@phdthesis{simchowitz2021statistical,
  title={Statistical complexity and regret in linear control},
  author={Simchowitz, Max},
  year={2021},
  school={University of California, Berkeley}
}

@InProceedings{beyondbernoulli,
  title={Beyond Bernoulli:Generating Random Outcomes that cannot be Distinguished from Nature},
  author={Dwork, Cynthia and Kim, {Michael P.} and Reingold, Omer and Rothblum, {Guy N.} and Yona, Gal},
  booktitle = {International Conference on Algorithmic Learning Theory},
  year={2022}
}

@inproceedings{parrilo2006polynomial,
  title={Polynomial games and sum of squares optimization},
  author={Parrilo, Pablo A},
  booktitle={Proceedings of the 45th IEEE Conference on Decision and Control},
  year={2006},
  organization={IEEE}
}

@book{jiawang,
	address = {Philadelphia, PA},
	author = {Nie, Jiawang},
	publisher = {Society for Industrial and Applied Mathematics},
	title = {Moment and Polynomial Optimization},
	year = {2023}}

@inproceedings{lee2018efficient,
  title={Efficient convex optimization with membership oracles},
  author={Lee, Yin Tat and Sidford, Aaron and Vempala, Santosh S},
  booktitle={Conference On Learning Theory},
  year={2018},
}

@book{Grotschel1993,
	author = {Gr\"otschel, Martin and Lov\'{a}sz, L\'{a}szl\'{o} and Schrijver, Alexander},
	title = {{Geometric Algorithms and Combinatorial Optimization}},
	year = {1993}
}

@inproceedings{
fishelson2025highdimensional,
title={High-Dimensional Calibration from Swap Regret},
author={Maxwell Fishelson and Noah Golowich and Mehryar Mohri and Jon Schneider},
booktitle={Conference on Neural Information Processing Systems},
year={2025},
}

@article{peng2025high,
  title={High dimensional online calibration in polynomial time},
  author={Peng, Binghui},
  journal={Symposium on the Foundations of Computer Science},
  year={2025}
}

@article{kernelOI,
	author = {Dwork, Cynthia and Hays, Chris and Immorlica, Nicole and Perdomo, Juan C and Tankala, Pranay},
	journal = {Conference on Learning Theory},
	title = {From Fairness to Infinity: Outcome-Indistinguishable (Omni) Prediction in Evolving Graphs},
	year = {2025}}

@article{okoroafor2025near,
	author = {Okoroafor, Princewill and Kleinberg, Robert and Kim, Michael P},
	journal = {arXiv preprint arXiv:2501.17205},
	title = {Near-Optimal Algorithms for Omniprediction},
	year = {2025}}

@InProceedings{omnipredictors,
  author =	{Gopalan, Parikshit and Kalai, Adam Tauman and Reingold, Omer and Sharan, Vatsal and Wieder, Udi},
  title =	{{Omnipredictors}},
  booktitle =	{Innovations in Theoretical Computer Science Conference},
  year =	{2022}
}

@article{noarov2025high,
  title={High-dimensional prediction for sequential decision making},
  author={Noarov, Georgy and Ramalingam, Ramya and Roth, Aaron and Xie, Stephan},
  journal={International Conference on Machine Learning},
  year={2025}
}

@article{lu2025sample,
  title={Sample efficient omniprediction and downstream swap regret for non-linear losses},
  author={Lu, Jiuyao and Roth, Aaron and Shi, Mirah},
  journal={Conference on Learning Theory},
  year={2025}
}

@article{daskalakis2009complexity,
  title={The complexity of computing a Nash equilibrium},
  author={Daskalakis, Constantinos and Goldberg, Paul W and Papadimitriou, Christos H},
  journal={Communications of the ACM},
  volume={52},
  number={2},
  pages={89--97},
  year={2009},
  publisher={ACM New York, NY, USA}
}

@article{chen2009settling,
  title={Settling the complexity of computing two-player Nash equilibria},
  author={Chen, Xi and Deng, Xiaotie and Teng, Shang-Hua},
  journal={Journal of the ACM (JACM)},
  volume={56},
  number={3},
  pages={1--57},
  year={2009},
  publisher={ACM New York, NY, USA}
}

@article{yakubovich1971s,
  title={S-procedure in nonlinear control theory},
  author={Yakubovich, Vladimir A},
  journal={Vestnik Leninggradskogo Universiteta, Ser. Matematika},
  year={1971}
}

@article{vovk2007k29,
	author = {Vovk, Vladimir},
	journal = {Theoretical Computer Science},
	title = {Non-asymptotic calibration and resolution},
	year = {2007}}

@article{sandroni2003calibration,
  title={Calibration with many checking rules},
  author={Sandroni, Alvaro and Smorodinsky, Rann and Vohra, Rakesh V},
  journal={Mathematics of Operations Research},
  year={2003},
}

@article{lehrer2001any,
  title={Any inspection is manipulable},
  author={Lehrer, Ehud},
  journal={Econometrica},
  year={2001},
}

@article{foster2021forecast,
	author = {Foster, Dean P and Hart, Sergiu},
	journal = {Journal of Political Economy},
	title = {Forecast hedging and calibration},
	year = {2021}}

@article{gopalan2022loss,
	author = {Gopalan, Parikshit and Hu, Lunjia and Kim, Michael P and Reingold, Omer and Wieder, Udi},
	journal = {Innovations in Theoretical Computer Science},
	title = {Loss minimization through the lens of outcome indistinguishability},
	year = {2023}}

@book{o2014analysis,
  title={Analysis of boolean functions},
  author={O'Donnell, Ryan},
  year={2014},
  publisher={Cambridge University Press}
}

@inproceedings{vovk2005defensive,
	author = {Vovk, Vladimir and Takemura, Akimichi and Shafer, Glenn},
	booktitle = {International Workshop on Artificial Intelligence and Statistics},
	title = {Defensive forecasting},
	year = {2005}}

@inproceedings{hebert2018multicalibration,
	author = {H{\'e}bert-Johnson, Ursula and Kim, Michael and Reingold, Omer and Rothblum, Guy},
	booktitle = {International Conference on Machine Learning},
	title = {Multicalibration: Calibration for the (computationally-identifiable) masses},
	year = {2018}}

@article{alvarez2012kernels,
  title={Kernels for vector-valued functions: A review},
  author={Alvarez, Mauricio A and Rosasco, Lorenzo and Lawrence, Neil D and others},
  journal={Foundations and Trends in Machine Learning},
  year={2012},
}

@article{nemirovski2010accuracy,
  title={Accuracy certificates for computational problems with convex structure},
  author={Nemirovski, Arkadi and Onn, Shmuel and Rothblum, Uriel G},
  journal={Mathematics of Operations Research},
  year={2010},
}

@article{sigecom,
  title={Turning defense into offense in {$O(\log 1/\epsilon)$} steps: Efficient constructive proof of the minimax theorem},
  author={Gabriele Farina},
  year={2026},
  journal={ACM SIGecom Exchanges},
}

@article{farina2024polynomial,
  title={Polynomial-Time Computation of Exact {$\Phi$}-Equilibria in Polyhedral Games},
  author={Farina, Gabriele and Pipis, Charilaos},
  journal={Advances in Neural Information Processing Systems},
  year={2024}
}
\bibliographystyle{abbrvnat}
}

\appendix

\section{Expected Variational Inequalities}\label{app:evi}

\subsection{Proof of Proposition~\ref{prop:mc reg bnd}}\label{sec:mc reg}

    From the definition of $S_t$, the norm of the utility gradient is at most $G' := \max_z \|S_t(z)\|_2 \le t GD$. Hence, by setting the learning rate $\eta = D/(G\sqrt{T})$, and using the known analysis of the online projected gradient ascent algorithm \citep{zinkevich2003online}, we can write
    \[
        \frac{\mathrm{Regret}_K}{K} \le 2 G'D \frac{\sqrt{K}}{K} \le 2 D^2 G \frac{t}{\sqrt{K}}.
    \]
    This shows that it is enough to set $K = t^2$ to obtain an error of $2 D^2 G$ in the solution of \eqref{eq:evi} and thus, by \cref{prop:multicalibration error}, a multicalibration error bounded by $2\|h\|_\cH \sqrt{D^2 G T}$ against any $h \in \cH$.

    To achieve a given average multicalibration error $\epsilon$, it is then necessary to run \cref{alg:df} for $T = D^2 G / \epsilon^2 $. The only step left to complete the proof is then to estimate the complexity of each iteration of the no-regret-based EVI solution algorithm.
    
    Each iteration involves evaluating $S_t$ at the current $p_k$, taking a gradient step, and projecting onto $\cZ$. The latter two operations require $\poly(d)$ time. Thus, at every time $t$ we need to perform an amount of work that is of order
    $
        t^2 (\mathrm{poly}(d) + \textsf{Eval}(S_t)),
    $
    where $\textsf{Eval}(S_t)$ denotes the cost of evaluating $S_t$. We distinguish two possibilities for evaluating $S_t$.
    \begin{itemize}
        \item In general, $S_t$ is defined as a sum of $t$ expectations. Each expectation is over the distribution $\cD_\tau$ that solved \eqref{eq:evi} at time $\tau$. As discussed above, the support at time $\tau$ is $K_\tau = \tau^2$. So, given our assumption that each evaluation of $\Gamma$ takes $\poly(d)$ time, we are left with a $\cO(t^3 \poly(d))$ time bound for each evaluation of $S_t$.

        Putting all the pieces together, we are left with a $\cO(t^5 \poly(d))$ time per iteration of \cref{alg:df}, leading to a $\cO(T^6 \poly(d))$ runtime to complete $T$ iterations. Plugging $T = D^2 G / \epsilon^2$ yields a bound of $\poly(d,D,G)\cdot\epsilon^{-12}$.
        
        \item When $r = \cO(\poly(d, D, G))$, the cost of evaluating $S_t(p)$ can be amortized, by accumulating over time the vector quantity
        \[
            m_t \coloneqq \sum_{\tau=1}^{t-1} \E_{p_\tau \sim \cD_\tau} \Phi(x_\tau, p_\tau) (z_\tau - p_\tau).
        \]
        Indeed, by definition of $S_t$, we have $S_t(p) = \Phi(x_t, p)^\top m_t$, leading to a cost of $\cO(rd)$ for each evaluation of $S_t$. To solve the EVI, we then require $t^2 \poly(d) m$ time.

        After every iteration $t$ of \cref{alg:df}, the quantity $m_t$ can be updated by summing the new expectation arising from $\cD_t$. Since the distribution has support $K = t^2$, such an update has a $\cO(t^2 \poly(d))$ cost. 

        In total, each iteration of \cref{alg:df} requires $t^2 \poly(d) m$ time, for a total of $T^3 \poly(d) r$ time. Plugging $T = D^2 G / \epsilon^2$ yields a bound of $\cO(\poly(d, D, G) \cdot r \cdot \epsilon^{-6})$ time to reach $\epsilon$ average multicalibration error.
    \end{itemize}
    Taking the minimum between the two cases yields the statement.

\subsection{Intuition Behind the Construction of \cite{zhang2025expected}}\label{sec:evi description}
As mentioned in \cref{sec:calibration}, significantly faster approaches for solving EVIs (and even harder generalizations) at $\mathrm{log}(1/\epsilon)$ rates have been very recently developed by leveraging a new constructive version of the minimax theorem \citep{farina2024polynomial,zhang2025expected,sigecom}.

We now give a high-level intuition for these methods; for more details, we refer the reader to the paper of~\citet{zhang2025expected}. These faster algorithms can be understood as operating in two phases:
\begin{itemize}
    \item First, they produce an extremely sparse set $\cP = \{p_1, \dots, p_K\} \subset \cZ$ of size \[
        K = \cO(\mathrm{poly}(d)\log(G,D,1/\epsilon)).
    \]
    
    \item Then, they search for a distribution over the $K$ points in $\cP$ that guarantees value $\eps$ for all $z \in \cZ$. This discrete distribution $(\lambda_1, \dots, \lambda_K)$ solving the EVI problem can be computed by solving the convex optimization problem
    \[
        \argmin_{\lambda \in \Delta^K} \max_{z\in\cZ} \sum_{k=1}^K \lambda_k S(p_k)^\top (z - p_k).
    \]
    over $\Delta^K$. This can be solved using standard convex optimization techniques in time polynomial in $K$, the diameter of $\cZ$ and $S$ \citep{Grotschel1993}.
\end{itemize} 

Conceptually, to produce the sparse support, the algorithm makes use of the ellipsoid algorithm to certify the emptiness of the set $\Omega \coloneqq \{z \in \cZ: S(p)^\top (z - p) > 0 ~~\forall p \in \cZ\}$. The emptiness of $\Omega$ is direct, since for any $z \in \cZ$, the constraint indexed by $p = z$ is trivially violated. By running the ellipsoid over $\Omega$ using $p=z$ as the separation oracle, the ellipsoid method is therefore able to produce a trace of violated constraints indexed by the ellipsoid centers $p_1=z_1, p_2=z_2, \dots, p_K=z_K$ that \emph{sparsely certifies} the emptiness of $\Omega$. The Farkas lemma then implies that a distribution over the constraints must be an EVI solution, justifying the soundness of the second step.

\subsection{Proof of Proposition~\ref{prop:mc ellipsoid bnd}}\label{sec:mc ellipsoid}

The ellipsoid-based algorithm for EVIs is able to solve the EVI problem with only $$\cO(\poly(d)\log(D, G, \log 1/\epsilon))$$ calls to the operator $S_t$, producing a distribution with support size bounded by the same quantity. Following the discussion in \cref{sec:mc reg}, we distinguish two cases for the cost of evaluating $S_t$.
    \begin{itemize}
        \item In general, $S_t$ is defined as a sum of $t$ expectations. Each expectation is over the distribution $\cD_\tau$ that solved \eqref{eq:evi} at time $\tau$. Since the support at time $\tau$ is $K_\tau = \tilde\cO(\log(t/\eps))$, given our assumption that each evaluation of $\Gamma$ takes $\poly(d)$ time, we are left with a $\cO(t \cdot \poly(d))$ time bound for each evaluation of $S_t$.

        Putting all the pieces together, we are left with a $\cO(t \log^2(t/\eps))$ time per iteration of \cref{alg:df}, leading to a $\tilde\cO(T^2 \poly(d))$ runtime to complete $T$ iterations. Plugging $T = D^2 G / \epsilon^2$ yields a bound of $\tilde\cO(\poly(d, D, G) \cdot \epsilon^{-4})$.
        
        \item When $r = \cO(\poly(d, D, G))$, the cost of evaluating $S_t(p)$ can be amortized as described in \cref{sec:mc reg}, leading to a $\cO(\poly(d) r)$ evaluation time for $S_t$. Hence, we can solve the EVI in time $\tilde\cO(\poly(d) r)$

        After every iteration $t$ of \cref{alg:df}, the quantity $S_t$ can be updated by summing the new expectation arising from $\cD_t$. Since the distribution has support $K = \tilde\cO(1)$, such an update has a $\tilde\cO(\poly(d))$ cost. 

        In total, each iteration of \cref{alg:df} requires $\tilde\cO(\poly(d) r)$ time, for a total of $\cO(T \poly(d) r)$ time. Plugging $T = D^2 G / \epsilon^2$ yields a bound of $\cO(\poly(d, D, G) \cdot r \cdot \epsilon^{-2})$ time to reach $\epsilon$ average multicalibration error.
    \end{itemize}
    Taking the minimum between the two cases yields the statement.

\section{More details on the Cone of the first $d$ Moments in $\R$}\label{app:moment cone}

In the interest of keeping the exposition as self-contained as possible, in this section we sketch classical results regarding the algebraic and computational properties of the cone of moments. For more details, we refer the reader to Chapters 3.3 and 3.4 of the book on moment and polynomial optimization by \citet{jiawang}, or the paper of \citet{parrilo2006polynomial}.

To lighten notation, in this section we use the notation $\langle \cdot,\cdot\rangle$ to denote the standard dot product in $\R^{2d+1}$.

\paragraph{Efficient Separation Oracle for the Moment Set}\label{sec:moment cone}

An efficient separation oracle for $\mathcal{Z}$ can be constructed by leveraging a duality result between moments and nonnegative polynomials, and invoking classical Positivstellensatz results to characterize the latter via the positive semidefinite (PSD) cone.
To start, it is well understood that the dual set $\mathcal{Z}^*$ of the $\mathcal{Z}$ from \eqref{eq:moment_cone}, defined as
\[
    \mathcal{Z}^* \coloneqq \{a = (a_0, a_1, \dots, a_{2d}) \in \mathbb{R}^{2d+1}: \langle a, z \rangle \ge 0 \quad\forall z \in \mathcal{Z}\}
\]
corresponds to the set of polynomials that are nonnegative everywhere on the interval $[-1,1]$:
\[
    \mathcal{Z}^* = \left\{(a_0, a_1, \dots, a_{2d}) \in \mathbb{R}^{2d+1}: \sum_{i=0}^{2d} a_i y^i \ge 0 \quad\forall y \in [-1,1]\right\}.
\]
To see this, observe that for any $a \in \mathbb{R}^{2d+1}$,
\[
    \langle a, z\rangle \ge 0 ~~\forall z \quad\iff\quad
    \sum_{i=0}^{2d} \left(a_i \int_{[-1,1]} y^i \di\mu\right) \ge 0~~\forall \mu
    \quad\iff\quad
    \int_{[-1,1]} \left(\sum_{i=0}^{2d} a_i y^i\right) \di\mu \ge 0~~\forall \mu.
\]
Since, in particular, we are free to take $\mu$ to be a Dirac delta centered at any $y \in [-1,1]$, it is immediate to see that
\[
\int_{[-1,1]} \left(\sum_{i=0}^{2d} a_i y^i\right) \di\mu \ge 0~~\forall \mu \quad\iff\quad \sum_{i=0}^{2d} a_i y^i \ge 0~~\forall y \in [-1,1],
\]
completing the proof. Both $\mathcal{Z}$ and $\mathcal{Z}^*$ are convex and closed sets.

The dual set $\mathcal{Z}^*$ plays a key role in constructing an efficient separation oracle for $\mathcal{Z}$. Indeed, one can prove that $\mathcal{Z}^*$ defines a \emph{cutting plane characterization} of $\mathcal{Z}$, in the sense that
\begin{equation}\label{eq:bidual}
    \mathcal{Z} = \{z \in \mathbb{R}^{2d+1} : z_0 = 1 \land \langle z, a\rangle \ge 0 \quad \forall a \in \mathcal{Z}^*\}.
\end{equation}
The proof of the above result is a standard application of separation for closed convex sets. \cref{eq:bidual} implies that a separation oracle for $\mathcal{Z}$ can be constructed directly from a linear optimization oracle for $\mathcal{Z}^*$. Specifically, to check whether a given point $w$ belongs to $\mathcal{Z}$, one can do the following:
\begin{itemize}
    \item if $w_0 \neq 1$, then clearly $w \notin \mathcal{Z}$, and the vector $(1, 0, \dots, 0)$ provides a separating direction; else
    \item We solve the convex optimization problem $\min_{a \in \mathcal{Z}^*} \langle w, a\rangle$. If the optimal value of the problem is non-negative, then $w \in \mathcal{Z}$; else, the minimizer $a^*$ provides a separating direction, as $\langle a^*, w\rangle < 0$ by assumption, and yet $\langle a^*, z\rangle \ge 0$ for all $z \in \mathcal{Z}$ by \Cref{eq:bidual}.
\end{itemize}

To complete the construction, it remains to show that the cone $\mathcal{Z}^*$ of nonnegative polynomials on $[-1,1]$ admits an efficient linear optimization oracle. As mentioned above, this follows from important results in semi-algebraic optimization. In particular, it is a celebrated result that in dimension one, polynomial nonnegativity is intimately connected with the notion of sum-of-squares (SOS) polynomials. A polynomial $q$ of degree at most $2d$ is said to be SOS, denoted $q \in \Sigma_{2d}$, if it can be written in the form
\begin{equation}\label{eq:def sos}
    q(y) = \begin{pmatrix}1 \\ y \\ \vdots \\ y^{d}\end{pmatrix}^\top Q \begin{pmatrix}1 \\ y \\ \vdots \\ y^{d}\end{pmatrix}, \text{ where } Q = \begin{pmatrix} q_{00} & q_{01} & \dots & q_{0d}\\ q_{01} & q_{11} & \dots & q_{1d} \\ \vdots & \vdots & \ddots & \vdots \\ q_{0d} & q_{1d} & \dots & q_{dd}\end{pmatrix}\succeq 0_{d+1} \text{ is a PSD matrix}.
\end{equation}

An application of the Positivstellensatz result for univariate polynomials states that a polynomial $p(y) = a_0 + \dots + a_{2d} y^{2d}$ is nonnegative on $[-1,1]$ if and only if it can be decomposed in the form 
\[
    p(y) = r(y) + (1-y^2) s(y), \quad \text{for } r \in \mathrm{\Sigma}_{2d},~~s \in \mathrm{\Sigma}_{2d - 2}.
\]
The coefficients of the polynomial on the right-hand side depend linearly on the PSD matrices $R$ and $S$ underlying $r$ and $s$. Letting $H : \mathbb{R}^{(d+1)\times (d+1)} \times \mathbb{R}^{d \times d} \to \mathbb{R}^{2d+1}$ denote the mapping from $(R,S)$ to the coefficients of the right-hand side polynomial, we therefore conclude that
\[
\mathcal{Z}^* = \left\{a \in \mathbb{R}^{2d+1} : a = H(R, S), ~~ R \succeq 0_{d+1}, ~~ S \succeq 0_{d} \right\}.
\]
Hence $\mathcal{Z}^*$ is a convex conic domain defined as the intersection between the linear constraint $a = H(R, S)$, and the product of two PSD cones. Minimization of a linear objective over $\mathcal{Z}^*$, as used by the separation oracle for $\mathcal{Z}$, can therefore be solved using standard semidefinite programming techniques.

\paragraph{Solving for $\mu_t$.} Having discussed how to implement a separation oracle for the moment code, we now describe how given a target vector of moments $z_t \in \cZ$, one can produce a discrete distribution $\mu$ over $[-1,1]$ with at most $2d + 1$ atoms in the support, whose moments match $z$, $$\E_{\mu_t}[s(\yt_t)]^\top = (\E_{\mu_t}[\yt_t], \dots, \E_{mu_t}[\yt_t^{2d}])=z_t.$$

To do so, we leverage once again the duality between the moment set $\mathcal{Z}$ and the set $\mathcal{Z}^*$, as captured by \eqref{eq:bidual}. In light of the connection, the statement that $z \in \mathcal{Z}$ is witnessed by the fact that $z_0 = 1$ and $\langle z, a\rangle \ge 0$ for all $a \in \mathcal{Z}^*$. In particular, let
\[
    a^* \in \argmin_{\substack{a \in \mathcal{Z}^*: ~ 1^\top a \ge 1}}~ \langle a, z\rangle,
\]
which can be computed efficiently by semidefinite programming using the ideas we discussed.\footnote{We remark that the feasible set is not empty, since all $a \in \mathcal{Z}^*$ must be such that $a^\top w \ge 0$ for all $w \in \cZ$, and $1 \in \cZ$ since it is the moment vector of the point mass distribution with support $\{1\}$.} Since $(a^*)^\top z \ge 0$, we can assume without loss of generality that the optimal solution satisfies $1^\top a^* = 1$. Since $a^*$ belongs in $\mathcal{Z}^*$, it induces a nonzero polynomial
\[
    p^*(y) \coloneqq a^*_0 + a^*_1 y + \dots + a^*_{2d} y^{2d}, \quad\qquad p^*(y) \ge 0 \quad\forall y \in [-1,1].
\]
Let $\{\gamma_1, \dots, \gamma_m\}$ the distinct zeros of $p^*$ (of course, $0 \le m \le 2d$, since $p^*$ is not identically zero). 

We now show that the set $\{\gamma_0 \coloneqq 1, \gamma_1, \dots, \gamma_m\}$ defines a valid discrete support for a distribution $\mu$ that matches the moments $z$. 
By the first-order optimality conditions, it must then be the case that the gradient of the objective is in the normal cone of the feasible set at $a^*$. Since all constraints are binding by construction, this condition is exactly the existence of convex combination coefficients $\lambda_0, \dots, \lambda_m$ such that
\begin{equation}\label{eq:support lp}
    z = \sum_{j=0}^m \lambda_j \begin{pmatrix}1 \\ \gamma_j^1 \\ \vdots \\ \gamma_j^{2d}\end{pmatrix}.
\end{equation}
In other words, there exists a discrete distribution supported on $\{1, \gamma_1, \dots, \gamma_m\}$ whose moments match $z$. The probability masses $\lambda_j$ can be computed directly by solving for the coefficients $\lambda_j$ in~\eqref{eq:support lp}, which can be done efficiently via a linear program.

\section{The Set of First and Second Moments on the Euclidean Ball}
\label{app:2_cone}
In this section, we provide similar results as in the previous appendix section but for the case where the outcomes $y$ live in the unit ball in $\R^d$. We start by recalling the definition of the moment set:
\begin{align}
    \cZ = \{(v,Q): \exists \; \text{a probability measure } \mu \text{ over } \|y\|_2 \leq 1 \text{ such that } \E_{\mu}[y]=v, \E_{\mu}[yy^\top] = Q\}
\end{align}
Since $\|y\|_2\leq 1$ and $\Tr[Q]=\E_{\mu}\|y\|_2^2$, it must be the case that $\Tr[Q]\leq 1$ and that $Q = \E_{\mu}[yy^\top] \succeq 0$. Furthermore, for $(v,Q)= (\E_{\mu}[y], \E_{\mu}[yy^\top])$, it is also true that:
\begin{align*}
    \E_{\mu}[(y-\E_{\mu}[y])(y-\E_{\mu}[y])^\top] \succeq 0 \iff Q - vv^\top = \E_{\mu}[yy^\top] - \E_{\mu}[y]\E_{\mu}[y]^\top \succeq 0.
\end{align*}
Using these probability facts, stipulating necessary conditions on $(v,Q) \in \cZ$, we conclude that $\cZ \subseteq \mathcal{M}$ where,
\begin{align*}
    \mathcal{M} = \{(v,Q): Q\succeq vv^\top ,Q\succeq0,\Tr[Q] \leq 1\}.
\end{align*}
Next, we show that given any $(v,Q)\in \mathcal{M}$ we can construct a probability measure $\mu$ over points $y$ in the unit ball using \Cref{alg:moment} such that $\E_{\mu}[y]=v$ and $\E_{\mu}[yy^\top] = Q$. This both solves the problem of finding a measure $\mu$ with appropriate moments and shows that $\mathcal{M} \subseteq \cZ$. Since we had already argued that $\cZ \subseteq \mathcal{M}$, this establishes that $\cZ = \mathcal{M}$ and solves the problem of characterizing the set $\cZ$ in terms of PSD constraints.

We establish the correctness of \Cref{alg:moment} in the following proposition. Similar results and constructions have appeared throughout the literature, we include the proof here simply for the sake of having a self-contained exposition.
\begin{proposition}
\label{prop:moment_sampling_2}
Let $(v,Q)$ be any element in $\mathcal{M}$. \Cref{alg:moment} returns an atomic measure $\mu$ supported on $2d+1$ points such that $\E_{\mu}[y]=v$ and $\E_{\mu}[yy^\top] =Q$.
\end{proposition}

\begin{proof}
Consider the SVD of $\Sigma = Q - vv^\top = \sum_{i=1}^d \sigma_i u_i u_i^\top$. We can assume without loss of generality that the $u_i$ form an orthonormal basis for $\R^d$ and that $\sigma_i \geq 0$ for all $i$. Also for $i\in [d]$, the equation, $$\|v + tu_i\|_2^2 =1 \iff t^2 + 2tv^\top u_i - (1- \|v\|^2)=0,$$ always has either 2 solutions $t_i^+$ and $t_i^-$ since $\|v\|^2 = \Tr[vv^T] \leq \Tr[Q] \leq 1$ or 1 solution (which is $t=0$) if $\|v\|=1$. If there are 2 solutions, by the quadratic formula $t_i^+t_i^-=-(1- \|v\|_2^2) \leq 0$. Hence we can let $t_i^+ >0$ and $t_i^-< 0$. We will deal with the case where there are 2 solutions that are non-zero ($\|v\|<1$) and address the other case later. We define
$$y_i^+ = v + t_i^+u_i, \;y_i^- = v + t_i^-u_i \quad \text{ and }\quad \lambda_i^+ =\frac{\sigma_i}{t_i^+(t_i^+- t_i^-)}>0, \lambda_i^- =\frac{\sigma_i}{t_i^-(t_i^-- t_i^+)}>0.$$
We can check that $\Lambda_i = \lambda_i^++\lambda_i^- = \sigma_i / (1-\|v\|_2^2)$. Recall that the measure is,
$$\mu = \lambda_0 \delta(v) + \sum_{i=1}^d[\lambda_i^+ \delta(y_i^+) +  \lambda_i^- \delta(y_i^-)],$$ where $\lambda_0 = 1-\sum_{i=1}^d \lambda_i$. By construction, the total weight amongst all of the $2d+1$ points in the measure is:
\begin{align*}
    \lambda_0 + \sum_{i=1}^d \Lambda_i = \lambda_0 + \frac{\sum_{i=1}^d\sigma_i}{1-\|v\|^2}= 1.
\end{align*}
Note that by definition of $(v,Q)$, $1- \|v\|_2^2\geq \Tr[Q]- \|v\|^2_2=\Tr[Q-vv^\top] = \sum_{i=1}^d \sigma_i \geq0$ hence the fraction is well-defined. Here we note that if $\|v\|=1$ (the other case) we are effectively just adding weight on the $v$ so there is no issue.

It remains to show that $\mu$ has the right first and second moments. A direct calculation shows that for every $i$, $\lambda_i^+t_i^+ + \lambda_i^-t^-_i =0$. Using this, and plugging in the definitions of $y_i^+, y_i^-$, we get that 
\begin{align*}
    \E_{\mu}[y] = \lambda_0 v + \sum_{i=1}^d(\lambda_i^+ y_i^+ + \lambda_i^-y_i^-) = \lambda_0 v + \sum_{i=1}^d(\lambda_i^+ + \lambda_i^{-})v  + \sum_{i=1}^d (\lambda_i^+ t_i^++ \lambda_i^- t_i^-) u_i = v
\end{align*}
To verify the second moment, we use the fact that $\lambda_i^+ (t_i^+)^2+ \lambda_i^- (t_i^-)^2 = \sigma_i$ and plug in again the definition of $\mu$ to get that:
\begin{align*}
    \E_{\mu}[(y-v)(y-v)^\top] = \sum_{i=1}^d (\lambda_i^+ (t_i^+)^2 + \lambda_i^- (t_i^-)^2) u_iu_i^\top = \sum_{i=1}^d \sigma_i u_i u_i^\top = Q-vv^\top.
\end{align*}
\end{proof}

\section{Further Details on the Examples in \Cref{sec:examples}}
\label{app:examples}

\paragraph{Predicting Correlated Rain Across the United States.} The particular guarantees stated there are essentially a restatement of \Cref{corr:Rd}. In particular, let $k((x,p), (x',p')) = x^\top x' + p^\top p'$ be the linear, scalar-valued kernel. Then, $\Gamma = I \cdot k$ is a matrix-valued kernel that contains all the functions, 
\begin{align*}
    h(x,p) = A x + C p 
\end{align*}
where $\|h\|_{\cH}^2 = \|A\|_{F}^2 + \|C\|_{F}^2$ and $\|\cdot\|_{F}$ denotes the Frobenius norm. Note that 
\begin{align}
    \|  \Gamma((x,p),(x,p)) \|_{\mathrm{op}} = \| I (\|x\|^2_2 +  \|p\|_2^2)  \|_{\mathrm{op}}
\end{align}
By Assumption $\|x\|_2^2 \leq B^2$. Also, since $p = (\E_{\mu_t}[\yt_t], \E_{\mu_t}[\yt_t\yt_t^\top])$
\begin{align*}
    \|p\|_2^2 = \|\E_{\mu_t}[\yt_t]\|_2^2 + \|\E_{\mu_t}[\yt_t \yt_t^\top]\|_{F}^2  
\end{align*}
The first term is at most 1 since $\yt_t$ is on the unit ball. By Jensen's, $\|\E_{\mu_t}[\yt_t \yt_t^\top]\|_{F}^2  \leq   \E_{\mu_t} \|[\yt_t \yt_t^\top]\|_{F}^2 \leq 1$ and $\|  \Gamma((x,p),(x,p)) \|_{\mathrm{op}} \leq B^2 + 2$.

Defensive Generation guarantees indistinguishability with respect to all functions of the form $f(x,p,y) = h(x,p)^\top s(y)$ where as in \Cref{eq:f_vec}  $s(y)$ is a vector containing $\E_{\mu_t}[\yt_t]$ and $\E_{\mu_t}[\yt_t \yt^\top_t]$. 

The specific examples in that section correspond to cases where $A$ and $B$ are zero everywhere except for one row corresponding to a specific entry in $s(y)$. In this case, the Frobenius norm of the matrices just becomes the $\ell_2$ norm of that row. The rest of the statement follows from \Cref{corr:Rd}.

\end{document}